%% file: main.tex
\documentclass[10pt, conference, letterpaper]{IEEEtran}
\IEEEoverridecommandlockouts
\usepackage{graphicx}
\usepackage{epstopdf}
\usepackage{cite}
\usepackage{amsmath,amssymb,amsfonts}
\usepackage{graphicx}
\usepackage{textcomp}
\usepackage{enumitem,kantlipsum}
\usepackage{comment}
\usepackage{url}
\usepackage{amsthm}
\usepackage{balance}

\usepackage{caption}

\usepackage{subcaption}
\usepackage{multirow}
\usepackage{setspace}

\usepackage{soul}

\usepackage{mathtools}
\usepackage{booktabs}

\newtheorem{thm}{Theorem}
\newtheorem{lem}{Lemma}
\newtheorem{cor}{Corollary}
\theoremstyle{remark}
\newtheorem{rem}{Remark}
\theoremstyle{definition}
\newtheorem{defi}{Definition}

\newtheorem{asmp}{Assumption}

\usepackage{fixltx2e}
\usepackage{accents}

\usepackage{mathtools}
\usepackage{bbm}
\usepackage{comment}
\usepackage{pifont}
\usepackage[table,xcdraw]{xcolor}
\usepackage[linesnumbered,ruled,vlined]{algorithm2e}

\usepackage[tr]{optional}

\MakeRobust{\underaccent}
\def\BibTeX{{\rm B\kern-.05em{\sc i\kern-.025em b}\kern-.08em
		T\kern-.1667em\lower.7ex\hbox{E}\kern-.125emX}}

\begin{document}
	
\title{TACO: Tackling Over-correction in Federated Learning with Tailored Adaptive Correction} 

	\author{\IEEEauthorblockN{Weijie Liu\textsuperscript{1,2}, Ziwei Zhan\textsuperscript{1}, Carlee Joe-Wong\textsuperscript{3}, Edith Ngai\textsuperscript{2}, Jingpu Duan\textsuperscript{4}, Deke Guo\textsuperscript{1}, Xu Chen\textsuperscript{1}, Xiaoxi Zhang\textsuperscript{1}}
		\IEEEauthorblockA{
			\textit{\textsuperscript{1}Sun Yat-sen University, \textsuperscript{2}The University of Hong Kong, \textsuperscript{3}Carnegie Mellon University, \textsuperscript{4}Pengcheng Laboratory} \\
			Email: liuwj0817@connect.hku.hk, zhanzw@mail2.sysu.edu.cn, cjoewong@andrew.cmu.edu, chngai@eee.hku.hk,\\
			duanjp@pcl.ac.cn, \{guodk, chenxu35, zhangxx89\}@mail.sysu.edu.cn} 
        \thanks{Xiaoxi Zhang is the corresponding author. This work was supported by NSFC grant (62472460), Guangdong Basic and Applied Basic Research Foundation (2024A1515010161, 2023A1515012982), Young Outstanding Award under the Zhujiang Talent Plan of Guangdong Province, UGC General Research Funds No. 17203320 and No. 17209822 from Hong Kong.}
 }
    \IEEEoverridecommandlockouts
	\maketitle
	\IEEEpubidadjcol
	\begin{abstract}

        Non-independent and identically distributed (Non-IID) data across edge clients have long posed significant challenges to federated learning (FL) training. Prior works have proposed various methods to mitigate this statistical heterogeneity. 
        While these methods can achieve good theoretical performance, they may lead to the over-correction problem, which degrades model performance and even causes failures in model convergence. In this paper, we provide the first investigation into the hidden over-correction phenomenon brought by the uniform model correction coefficients across clients adopted by the existing methods. To address this problem, we propose TACO, a novel algorithm that addresses the non-IID nature of clients' data by implementing fine-grained, client-specific gradient correction and model aggregation, steering local models towards a more accurate global optimum. Moreover, we verify that leading FL algorithms generally have better model accuracy in terms of communication rounds rather than wall-clock time, resulting from their extra computation overhead imposed on clients. To enhance the training efficiency, TACO deploys a lightweight model correction and tailored aggregation approach that requires minimum computation overhead and no extra information beyond the synchronized model parameters. To validate TACO's effectiveness, we present the first FL convergence analysis that reveals the root cause of over-correction. 
        Extensive experiments across various datasets confirm TACO's superior and stable performance in practice.
	\end{abstract}
	
	
	\input{1_intro.tex}
	\input{2_preliminary.tex}

	\input{3_reevaluation.tex}
	\input{4_alg.tex}
	\input{5_experiment.tex}
\section{Related Work}

\label{sec:related}

To address the {\rm client drift} issue brought by non-IID client data in FL systems, researchers have extensively developed innovative methods based on FedAvg~\cite{mcmahan2017communication} in different stages of FL training. These methods fall into three categories based on the different correction terms employed. 
The first category of pioneering works includes FedProx~\cite{li2020federated}, FedDyn~\cite{acar2021federated}, and FedDC~\cite{gao2022feddc} which mitigate client drift by adding a regularization term to the local loss function. 
The second approach focuses on global aggregation function calibration, \textit{e.g.}, FoolsGold~\cite{fung2020limitations}, FedNova~\cite{wang2020tackling}, and Ruan \textit{et al.}~\cite{ruan2021towards} refine the global aggregation function to ease client drift. 
The last popular mainstream leverages momentum-based model updates by utilizing the global gradient or model parameters to rectify local updates, such as Scaffold~\cite{karimireddy2020scaffold}, STEM~\cite{khanduri2021stem}, FedACG~\cite{fedacg}, FedMoS~\cite{wang2023fedmos}, Mime \cite{karimireddy2020mime}, and ~\cite{karimireddy2020mime,blanchard2017machine, haddadpour2021federated}. Scaffold~\cite{karimireddy2020scaffold} estimates and utilizes a variance reduction term as the momentum to help correct client-shift during the local update steps, and STEM~\cite{khanduri2021stem} further improves this momentum-based method by incorporating the benefits of both local and global momentum in the server and clients. TACO enhances the combination of the latter two approaches through tailored correction and tailored model aggregation. Nonetheless, the methodologies of the first approach are orthogonal to ours, suggesting that they could also benefit from integrating our tailored correction coefficients design in TACO to further refine their techniques. We have shown the enhancements in FedProx and Scaffold after the application of our tailored correction design in TACO (Fig. \ref{fig:extension}).

\section{Conclusion}

\label{sec:conclusion}

In this work, we thoroughly re-evaluate pioneering FL algorithms under non-IID data settings and disclose their existing weaknesses during the training, including instability and unsatisfactory time-to-accuracy performance. We unveil an over-correction phenomenon embedded in these algorithms to explain their suboptimal performance and instability during the FL training. To address these challenges, we propose a full-fledged and lightweight FL algorithm named TACO. It can implement tailored adaptive corrections for different clients to address the over-correction problem under non-IID data, while providing a certain level of robustness to freeloaders without imposing excessive computation burden on clients. Theoretical analysis and extensive experiments demonstrate the superiority, adaptability, and scalability of TACO.
        \input{8_appendix}
	\clearpage
	\bibliographystyle{IEEEtran}
	\bibliography{main.bib}
\end{document}

%% file: 1_intro.tex
\section{Introduction}

With the exponential growth in data generation from geographically distributed locations and the proliferation of edge computing technologies~\cite{zhou2019edge}, federated learning (FL)~\cite{mcmahan2017communication, liao2023accelerating,dinh2020federated,xia2024accelerating} has gained much attention as it enables distributed model training through multiple collaborative edge devices without exposing their raw data. 
Each edge device (a.k.a. client) in FL is allowed to perform multiple updates using local datasets before uploading their gradients to the central server. 
Despite its advantages, skewed label distributions and unbalanced sizes of training data across different clients, \textit{i.e.}, non-IID client data, have always been the major obstacle in FL implementation. These discrepancies create inconsistencies between local and global data distributions, incurring inevitable model drift during the local updates performed at edge FL clients. 

\vspace{0.5mm}
\noindent{\bf Over-correction and instability issues.} Previous studies have developed various innovative FL algorithms under non-IID data distribution, including FedProx~\cite{li2020federated}, FoolsGold~\cite{fung2020limitations}, Scaffold~\cite{karimireddy2020scaffold}, STEM~\cite{khanduri2021stem} and FedACG~\cite{fedacg}. These algorithms alleviate the client drift problem by introducing novel designs of model update correction based on FedAvg~\cite{mcmahan2017communication} in different stages of the training process (Algorithm \ref{alg:summary}). Their basic idea of model update correction is to rectify the local model by leveraging the global model parameters or globally updated gradients as correction terms, and steer the local model towards a global optimum rather than a local optimum. However, recent experimental studies~\cite{li2022federated} present two valuable findings: (i.) None of these existing FL algorithms consistently outperform others in all non-IID cases, and sometimes even underperform FedAvg. (ii.) These algorithms exhibit greater instability compared with FedAvg during the training process. 

These findings motivate us to re-evaluate these classic FL algorithms (Section \ref{sec:reeval}), investigating the causes of the unsatisfactory model performance and the actual effectiveness of their model correction methods. 
In this work, we uncover a prevalent over-correction phenomenon that lies in the application of a {\em uniform} coefficient of the correction term for all participating clients in existing algorithms. Their delicate design to alleviate client drift problems can then become ineffective, resulting in compromised model performance and increased instability. Inspired by this, we propose to design a tailored and adaptive correction design for different clients, with the goal of eliminating model over-correction and achieving better FL training performance. 



\vspace{0.5mm}
\noindent{\bf Computation efficiency.} Further, we noticed that prior FL research often use round-to-accuracy~\cite{fedacg} as the key evaluation metric, which was defined as either the model accuracy after completing a certain number of communication rounds or the number of communication rounds required to reach a certain target accuracy. Yet, we believe time-to-accuracy (wall-clock training time to achieve a target accuracy) can more accurately characterize the training efficiency. 
The reason is that 
FL algorithms may vary in their wall-clock runtime per round due to different methods of computing correction terms. 
For instance, STEM~\cite{khanduri2021stem} excels in round-to-accuracy but struggles with time-to-accuracy owing to computing auxiliary parameters and extra gradients at edge clients, which often have limited computation capability. 
Therefore, we 
aim to evaluate both the number of communication rounds (round-to-accuracy) and local training time (time-to-accuracy) of different FL algorithms to reach the same target accuracy. To this end, a thorny problem arises: {\em How to implement an effective method to eliminate over-correction issues without sacrificing computation efficiency for FL training?} 
To realize our computation-efficient FL algorithm with tailored model update corrections
, we make the following technical contributions:
\begin{enumerate}[wide, labelwidth=!, labelindent=0pt]
\item \emph{Insightful re-evaluation of pioneering FL algorithms (Section \ref{sec:reeval})}. We present an in-depth analysis of six leading FL algorithms, assessing their motivation and designs in addressing non-IID data challenges. This re-evaluation highlights their shortcomings, \textit{i.e.}, degraded model performance due to over-correction and unsatisfactory time-to-accuracy performance.

\item \emph{Lightweight FL algorithm with tailored adaptive correction and convergence analysis (Section~{\ref{sec:alg}).}} We propose a novel FL algorithm \textbf{TACO} to address the over-correction issue arising from non-IID data, by employing tailored correction coefficients for individual clients. TACO requires only the uploaded local gradients from participating clients and does not need auxiliary parameters like those in~\cite{karimireddy2020scaffold,wang2020tackling}, thus alleviating computation and communication burdens on edge clients. We also derive the convergence rate of TACO to support our tailored correction design in tackling data heterogeneity. To the best of our knowledge, this is the first analysis that manages to quantify the negative effect of over-correction on model convergence through theoretical analysis and proves that the primary cause stems from the application of uniform model correction coefficients as seen in our evaluated leading FL methods. We also propose a mechanism for detecting freeloaders in TACO, reducing its dependence on client authenticity. 

\item \emph{Comprehensive experiments across diverse datasets, baselines, and settings (Section~{\ref{sec:exp}).}} We conduct a comprehensive suite of experiments spanning eight different datasets and compare against six pioneering baselines to showcase the superiority and high efficiency of TACO. Experiment results affirm TACO's advantages in non-IID data settings. 

\end{enumerate}
\vspace{-2.5mm}

%% file: 2_preliminary.tex
\section{Preliminaries}
\label{sec:pre}


We consider the commonly adopted parameter-server (PS) architecture for FL. It consists of $N$ distributed edge devices (or clients, defined as a set $\mathcal{N}$) with potentially limited computation resources, and a centralized PS for global aggregation. Each client $i\in \mathcal{N}$ has a local data set $\mathcal{D}_i$ with $D_i = |\mathcal{D}_i|$ data samples $\mathbf{x}_i=[\mathbf{x}_{i,1}, \mathbf{x}_{i,2}, ...,\mathbf{x}_{i,D_i}]$, which is non-IID across $i$.

Let $l(\mathbf{w},\mathbf{x}_{i,j})$ denote the loss function for each sample $\mathbf{x}_{i,j}$, and the local loss function of client $i$ is formulated as:
\begin{equation}\label{eq:local_loss}
	f_i(\mathbf{w})=\frac{1}{D_i}\sum\limits_{\mathbf{x}_{i,j}\in \mathcal{D}_i}l(\mathbf{w},\mathbf{x}_{i,j}).
\end{equation}

The ultimate goal is to train a shared (global) model $\mathbf{w}$ that minimizes the global loss function, defined as: 

\begin{equation}
	f(\mathbf{w})=\sum\limits_{i\in\mathcal{N}}\frac{D_if_i(\mathbf{w})}{D},
        \label{eq:global_loss}
\end{equation}
where 
$D=\sum_{i\in\mathcal{N}}D_i$ represents the number of the entire dataset over all clients. We consider the typical mini-batch FL setting, where each client $i$ uniformly at random samples a mini batch of data, denoted by $\xi_{i,k}^t$ from $\mathcal{D}_i$, to perform the $k^{th}$ local update step in the $t^{th}$ communication round (aggregation round), using the mini-batch SGD algorithm~\cite{shi2022talk}. 
The mini-batch loss function $f_{i}(\mathbf{w},\xi_t)$ of each client $i$ is:
\begin{equation}
	f_{i}(\mathbf{w},\xi_{i,k}^t)=\frac{1}{s}\sum\limits_{\mathbf{x}_{i,j}\in\xi_{i,k}^t}l(\mathbf{w}, \mathbf{x}_{i,j}),
 \label{eq:batch_loss}
\end{equation}
where 
$s$ is the size of the selected mini-batch. With a local learning rate $\eta_l>0$, the rule of updating the local model of client $i$ at the $k^{th}$ local update step in round $t$ is:
\begin{equation}\label{eq:local_update}
	\mathbf{w}_{i,k+1}^t = \mathbf{w}_{i,k}^t - \eta_l g_{i,k}^t, k \in [K],
\end{equation}
where $g_{i,k}^t\triangleq \nabla f_{i}(\mathbf{w}_{i,k}^t,\xi_{i,k}^t)$ and $[K]\triangleq\{0,\cdots, K-1\}$ represents the number of local update steps in every aggregation round. Clients will upload the accumulated local gradient $\Delta_{i}^t$ to the PS after $K$ local update steps, where:
\begin{equation}
    \label{eq:Delta_i^t}
    \Delta_i^t = \mathbf{w}_{i,0}^t - \mathbf{w}_{i,K}^t. 
\end{equation}
In the aggregation procedure, the PS first computes the aggregated global gradient $\Delta_{t+1}$ after collecting all the local gradients $\Delta_i^t$ from FL clients. Then, it will update the global model $\mathbf{w}_{t+1}$ with a global learning rate $\eta_g$:

\begin{equation}
    \Delta_{t+1} = \sum\limits_{i\in\mathcal{N}}p_i \Delta_i^t, \quad\mathbf{w}_{t+1} = \mathbf{w}_t - \eta_g\Delta_{t+1},
    \label{eq:aggregation}
\end{equation}
where $p_i$ is the aggregation weight for each client $i$, which often equals $\frac{1}{N}$ or $\frac{D_i}{D}$, depending on the FL developer~\cite{mcmahan2017communication}. 


%% file: 3_reevaluation.tex
\section{Re-Evaluation of Existing Methods}
\label{sec:reeval}

In this section, we summarize and re-evaluate six pioneering FL methods to expose the over-correction phenomenon and its detrimental effects on the training. Then, we present the unsatisfactory time-to-accuracy performance of these methods. 

\subsection{Representative FL Algorithms for Non-IID Data}
\label{sec:pre_alg}
We outline five FL algorithms (FedProx~\cite{li2020federated}, FoolsGold~\cite{fung2020limitations}, Scaffold~\cite{karimireddy2020scaffold}, STEM~\cite{khanduri2021stem}, and FedACG~\cite{fedacg}) in Algorithm \ref{alg:summary}. We use different colors to highlight their key differences compared to the vanilla FedAvg~\cite{mcmahan2017communication} in different stages of the FL training, based on which we summarize three different categories of FL approaches to tackle the non-IID challenges.

\textbf{Loss function regularization:} \textbf{FedProx} refines the loss function (Line \ref{line:fedprox}) by adding an L2-regularization term $\frac{\zeta}{2}\left\|\mathbf{w}-\mathbf{w}_t\right\|^2$ in every round $t$, where $\zeta$ is the hyper-parameter to control the weight of the regularization. This can efficiently limit the distance between the local and global models so as to mitigate the client drift brought by non-IID data. 

\textbf{Aggregation calibration: FoolsGold (FG)} presents a novel aggregation weight design  ($p_i$ in (\ref{eq:aggregation})) in the global aggregation process. Instead of using the uniform weight $(1/N)$ or the data-quantity $(D_i/D)$ weight design adopted by most FL algorithms, FG computes the aggregation weights based on the similarity between the local gradient and global gradient, \textit{i.e.} $\rho_i=\frac{\Delta_{t+1}^{\top}\Delta_i^t} {||\Delta_{t+1}||\cdot||\Delta_i^t||}$ (Line \ref{line:fedavg_agg}), tackling the non-IID data problem by reducing biases in the aggregation process.

\textbf{Momentum-based correction: Scaffold} proposes a momentum-based correction method in the local update process. Clients will add a momentum term $c_t - c_i^t$ in every local update step (Line \ref{line:scaffold}) to help mitigate the client drift problem due to the non-IID data, where $c_i^t = c_i^{t-1} - c_{t-1} +\Delta_i^{t-1}/K\eta_l$, and $c_t = c_{t-1} + \frac{1}{N}\sum_{i\in\mathcal{N}}(c_i^t - c_i^{t-1})$. Further, \textbf{STEM} proposes a two-sided momentum-based correction method, by adding momentum terms $\mathbf{v}_{i,K-1}^t$ in both local updates (Line \ref{line:stem_local}) and global aggregation steps (Line \ref{line:stem_agg}). \textbf{FedACG}, the latest state-of-the-art (SOTA) FL method~\cite{fedacg}, combines the design of FedProx and STEM algorithms by utilizing the momentum terms $m_t$ in both loss function (Line \ref{line:fedacg_loss}) and global aggregation (Line \ref{line:fedacg_global}) to address the non-IID data problem.

 Despite the apparent dissimilarities among the aforementioned algorithms, their solutions to address non-IID data problem share a core motivation which we encapsulate as \textbf{`Global guides Local':} whether it be through the use of the \textbf{global} model (FedProx), \textbf{global} aggregation weight (FoolsGold), or \textbf{global} aggregated gradients (Scaffold, STEM and FedACG) to help alleviate the client drift problem in \textbf{local} model training. However, they either adopt uniform correction coefficients ($\zeta$ in FedProx, $\alpha$ in Scaffold, $\alpha_t$ in STEM and $\beta$ in FedACG) or do not have any local correction (FoolsGold), which can easily lead to the under-, and often, over-correction phenomenon, as illustrated in Fig.~\ref{fig:tailored} and further explained in Section~\ref{subsec:over-correction}. 

\begin{figure}[t]
    \centering
    \includegraphics[width=8.5cm]{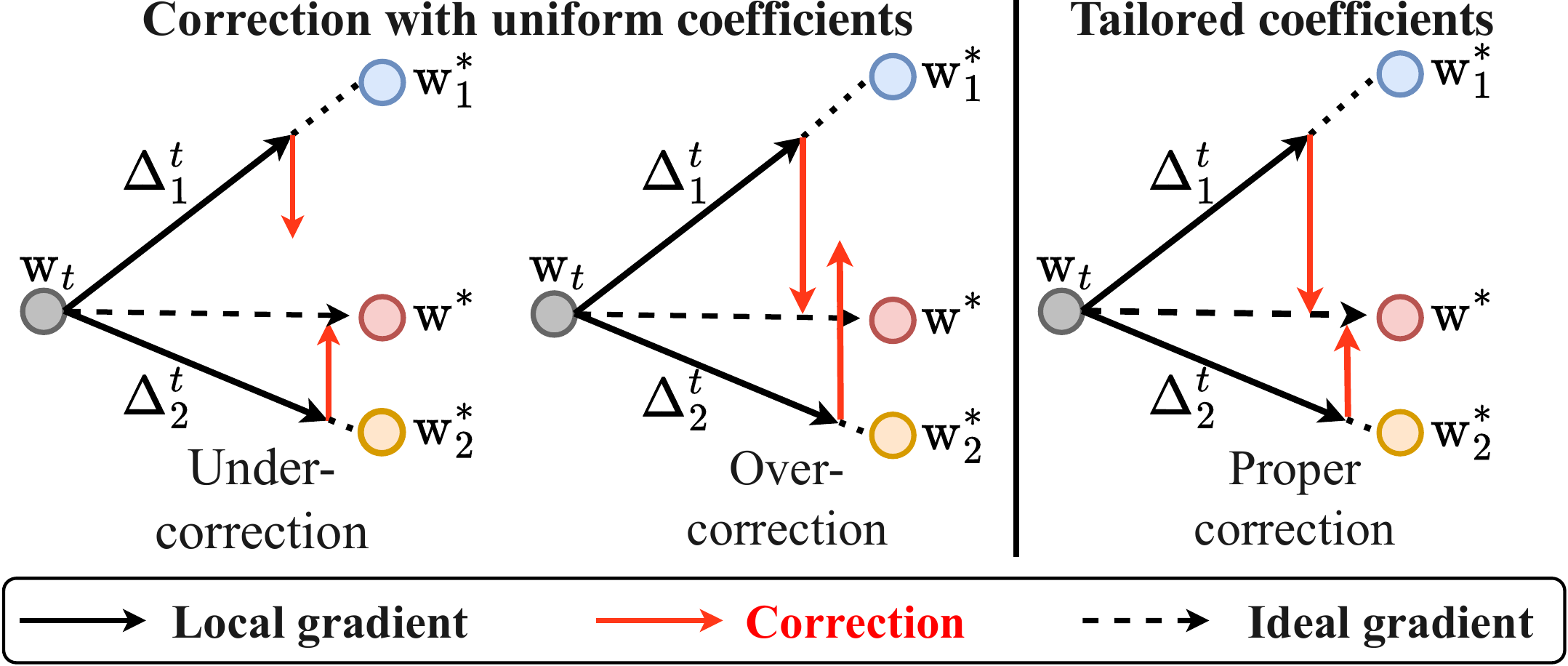}
    \caption{Federated Learning with non-IID data: Uniform correction coefficients versus tailored correction coefficients}
    \label{fig:tailored}
    \vspace{-1.5em}
\end{figure}
\newcommand{\highlighta}[1]{%
  \colorbox{gray!8}{\parbox{200pt}{#1}}%
}
\newcommand{\highlightb}[1]{%
  \colorbox{gray!8}{\parbox{188pt}{#1}}%
}
\newcommand{\highlightc}[1]{%
  \colorbox{gray!8}{\parbox{220pt}{#1}}%
}
\begin{algorithm}[ht]
 \SetKwData{Left}{left}\SetKwData{This}{this}\SetKwData{Up}{up} \SetKwFunction{Union}{Union}\SetKwFunction{FindCompress}{FindCompress} \SetKwInOut{Input}{Input}\SetKwInOut{Output}{Output}
	\Input{Step size $\eta_l, \eta_g$, {\color{cyan}$\eta_t$}, number of local updates $K$, aggregation round $T$}
	\Output{$\mathbf{w}_T$}
	\SetKwInOut{Initialize}{Initialize}
	\Initialize{$\mathbf{w}_0$, \textcolor{magenta}{$c_0, c_i^0$}, \textcolor{blue}{$m_0$}}
	\For{$t=0, 1, ..., T-1$}{
	    
	    	\For{each client $i \in \mathcal{N}$ \bf{in parallel}}{
                    $\mathbf{w}_{i,0}^t \leftarrow \mathbf{w}_t$ \;
                    \highlighta{\textcolor{blue}{FedACG}:$f_i(\mathbf{w},\xi)\leftarrow (\ref{eq:batch_loss}) + \frac{\beta}{2}\left\|\mathbf{w}-\mathbf{w}_t-m_t\right\|^2$\;\label{line:fedacg_loss}
                    \textcolor{red}{FedProx}:$f_i(\mathbf{w},\xi)\leftarrow (\ref{eq:batch_loss}) + \frac{\zeta}{2}\left\|\mathbf{w}-\mathbf{w}_t\right\|^2$\; \label{line:fedprox}
                    {\bf Others}: $f_i(\mathbf{w},\xi) \leftarrow (\ref{eq:batch_loss})$ \;}
                    \For{$k=0,1,...,K-1$}{
                        $g_{i,k}^t= \nabla f_i(\mathbf{w}_{i,k}^t, \xi_{i,k}^t)$\;
                        \highlightb{\textcolor{magenta}{Scaffold}:
                        $\mathbf{v}_{i,k}^t \leftarrow g_{i,k}^t + \alpha(c_{t} - c_i^t)$\label{line:scaffold}\;
                        \textcolor{cyan}{STEM}:
                        $\mathbf{v}_{i,k}^t \leftarrow g_{i,k}^t + (1-\alpha_t)(\mathbf{v}_{i,k-1}^{t} - \nabla f_i(\mathbf{w}_{i,k-1}^t, \xi_{i,k}^t))$\label{line:stem_local}\;
                        {\bf Others}:\,
                        $\mathbf{v}_{i,k}^t \leftarrow g_{i,k}^t$\;}
                        $\mathbf{w}_{i,k+1}^t \leftarrow \mathbf{w}_{i,k}^t - \eta_l \mathbf{v}_{i,k}^t$\;
                    }
                $\Delta_i^t \leftarrow \mathbf{w}_{i,0}^t - \mathbf{w}_{i,K}^t$\;
                Update and send $\Delta_i^t, {\color{magenta}c_i^{t+1}}, {\color{cyan}\mathbf{v}_{i,K-1}^t}$ to server\;
		}

        Update auxiliary parameters {\color{blue}$m_{t+1}$}, {\color{magenta}$c_{t+1}$}, {\color{orange}$\rho_i$}, {\color{cyan}$\alpha_{t+1}$}\;
	\highlightc{\textcolor{orange}{FoolsGold}:
		 $\Delta_{t+1} \leftarrow \frac{1}{KN\eta_l} \sum_{i=1}^N\frac{\rho_i\Delta_i^t}{\sum_{i=1}^N\rho_i}$ \label{line:foolsgold}\;
       \textcolor{cyan}{STEM}:
		 $\Delta_{t+1} \leftarrow \frac{1}{KN\eta_l} \sum_{i=1}^{N}(\Delta_i^t + \mathbf{v}_{i,K-1}^t)$ \label{line:stem_agg}\;
        \textcolor{blue}{FedACG}: $\Delta_{t+1} \leftarrow \frac{1}{D\eta_l} \sum_{i=1}^{N}D_i\Delta_i^t +\frac{m_{t+1}}{\eta_g}$\;\label{line:fedacg_global}
        {\bf Others}:
        $\Delta_{t+1} \leftarrow (i) \frac{1}{KN\eta_l} \sum_{i=1}^{N}\Delta_i^t ,\, (ii) \sum_{i=1}^{N}\frac{D_i\Delta_i^t}{D\eta_l}$\; \label{line:fedavg_agg}}
		$\mathbf{w}_{t+1} \leftarrow \mathbf{w}_t - \eta_g\Delta_{t+1}$ \; 
   
  Send $\mathbf{w}_{t+1}$,\,{\color{magenta}$c_{t+1}$},\,{\color{cyan}$\alpha_{t+1}$},\, {\color{blue}$m_{t+1}$} to clients

	}
	\caption{FedAvg/\textcolor{red}{FedProx}/\textcolor{orange}{FG}/\textcolor{magenta}{Scaffold}/\textcolor{cyan}{STEM}/\textcolor{blue}{FedACG}}
  \label{alg:summary}
  \vspace{-0.28em}
\end{algorithm}
\subsection{The Over-correction Phenomenon in Federated Learning}\label{subsec:over-correction}
We further illustrate how uniform correction coefficients can cause over-correction in Fig.~\ref{fig:tailored}. The optimal local model parameters $\mathbf{w}_1^*$ and $\mathbf{w}_2^*$ can deviate significantly from the global optimum $\mathbf{w}^*$ due to the skewed data distribution, leading to inevitable client drift during the local updates and impede model convergence in the absence of local model correction schemes, 
such as FedAvg~\cite{mcmahan2017communication} and FoolsGold~\cite{fung2020limitations}. 

Prior works have proposed various methods to address this drift issue, as analyzed in Section~\ref{sec:pre_alg}, and they utilized auxiliary parameters 
to control the correction terms in the local updates. Although these works design client-specific correction terms, \textit{e.g.}, $c_t - c_i^t$ in Scaffold, their correction coefficients (\textit{e.g.}, factors $\zeta$ in FedProx and $\alpha$ in Scaffold shown in Algorithm~\ref{alg:summary}) are uniform across clients and can easily result in under-correction or, more detrimentally, over-correction. As shown in Fig.~\ref{fig:tailored}, client 1's training data exhibits a higher non-IID degree compared to client 2. Therefore, setting a uniform correction factor that suits client 2 results in under-correction for client 1. Conversely, a uniform factor tailored to client 1 leads to over-correction for client 2. Since training data of different clients often have varying non-IID degrees, and the distance between the local optimum $\mathbf{w}_i^*$ and global optimum $\mathbf{w}^*$ varies, adopting a uniform correction factor for all clients is not the most effective way to correct local models towards the global optimum simultaneously. 

\begin{figure}
	\centering
	\begin{subfigure}[t]{0.48\linewidth}
		\includegraphics[width=\textwidth]{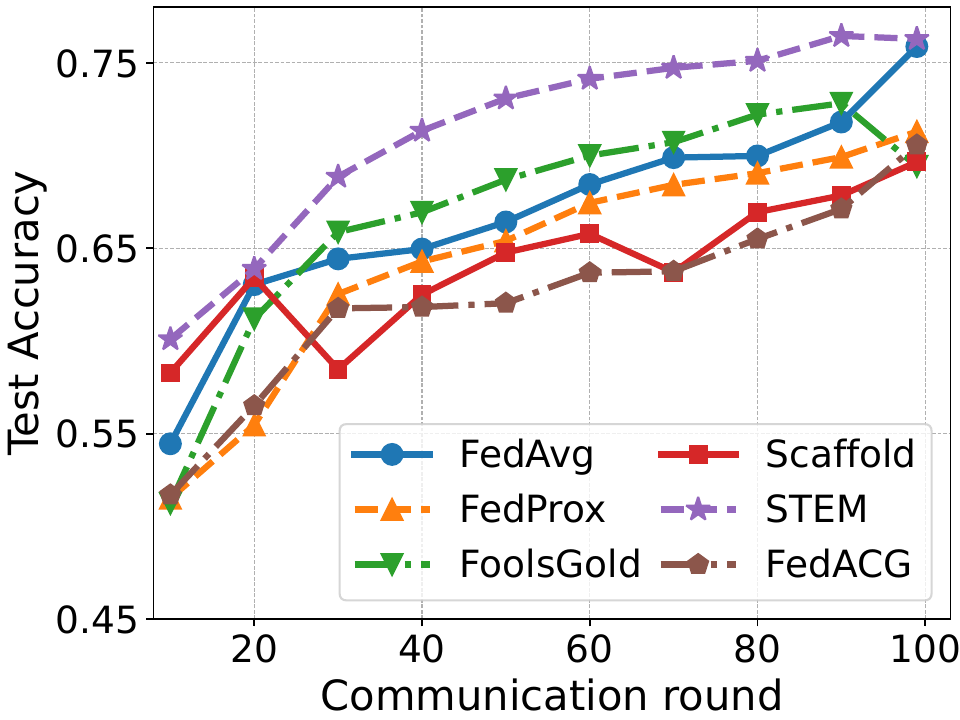}
		\caption{Round-Accuracy (FMNIST)}
            \label{fig:reeval-round-fmnist}
	\end{subfigure}
	\begin{subfigure}[t]{0.48\linewidth}
		\includegraphics[width=\textwidth]{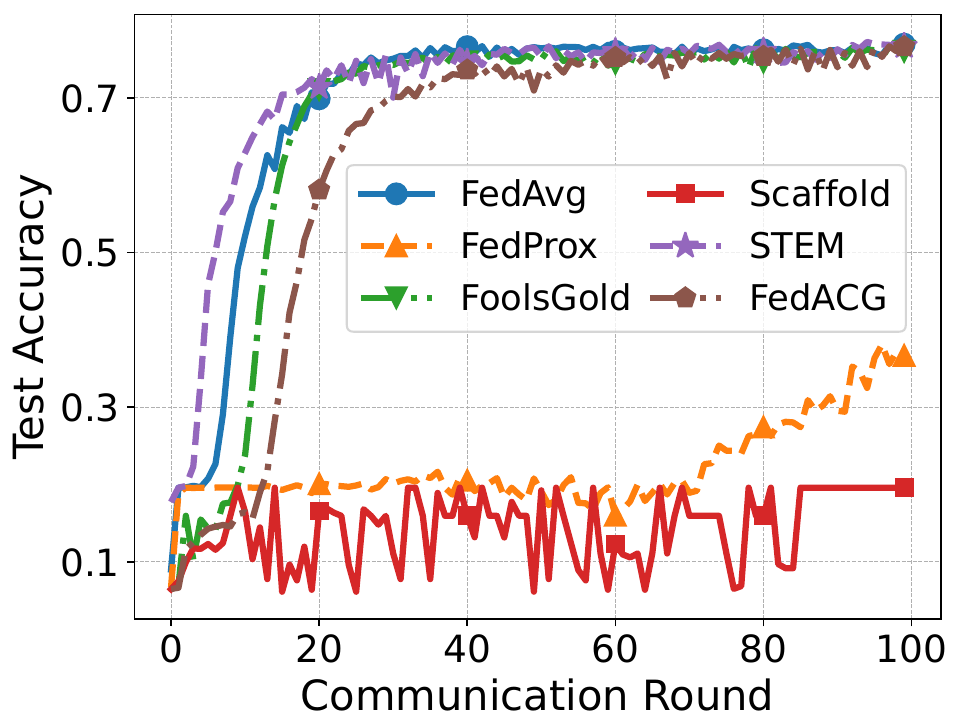}
		\caption{Round-Accuracy (SVHN)}
  \label{fig:reeval-round-svhn}
	\end{subfigure}
        \begin{subfigure}[t]{0.48\linewidth}
		\includegraphics[width=\textwidth]{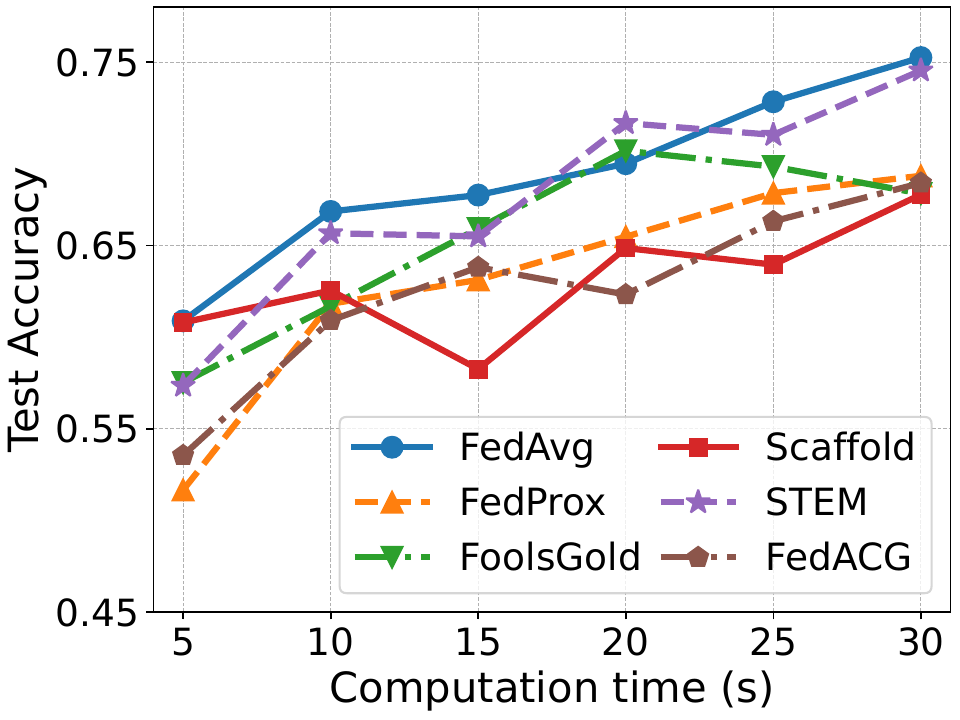}
		\caption{Time-Accuracy (FMNIST)}
  \label{fig:reeval-time-fmnist}
	\end{subfigure}
        \begin{subfigure}[t]{0.48\linewidth}
		\includegraphics[width=\textwidth]{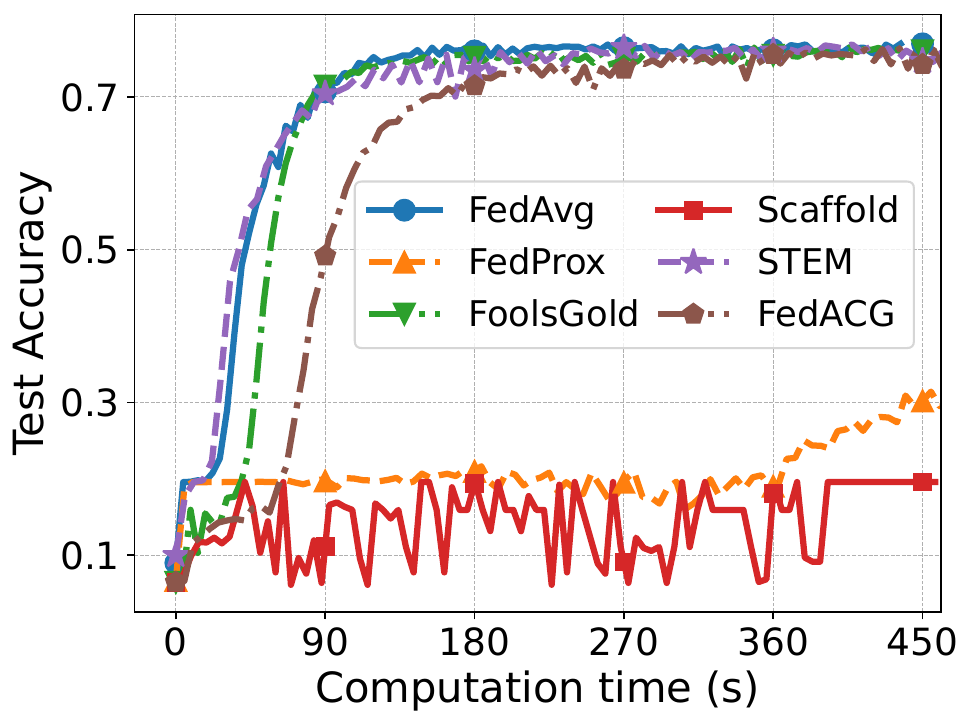}
		\caption{Time-Accuracy (SVHN)}
  \label{fig:reeval-time-svhn}
	\end{subfigure}
 
	\caption{Round-to-accuracy and time-to-accuracy re-evaluations}
	\label{fig:reeval-1}
\vspace{-2mm}
\end{figure}

We then conduct experiments under FMNIST and SVHN~\cite{netzer2011reading} datasets to validate our intuition. The re-evaluation experiments follow the experimental setup presented in Section \ref{sec:exp}. In order to exclude the potential impact of misconfigured hyper-parameters, we conduct experiments with multiple hyper-parameter configurations for different baseline algorithms, and present their best results in Fig.~\ref{fig:reeval-1}. Specifically, the candidate values of $\zeta$ (FedProx) and $\beta$ (FedACG) are \{0.001, 0.01, 0.1\}, the initial value of $\alpha_t$ in STEM is tuned from the set \{0.05, 0.1, 0.2\}, and the $\alpha$ in Scaffold is consistently set to 1, adhering to the original work. Figs. \ref{fig:reeval-round-fmnist} and \ref{fig:reeval-round-svhn} show that most algorithms except STEM cannot outperform FedAvg under these two datasets. FedProx and Scaffold have lower test accuracies and more severe instability compared to the vanilla FedAvg due to the over-correction brought by their applied uniform correction coefficients. Particularly with the SVHN dataset, both of these algorithms even fail to achieve model convergence, whereas FedAvg and FoolsGold, which do not incorporate any correction terms in their local updates, successfully complete training. We alleviate this over-correction issue by improving 
 FedProx and Scaffold with our tailored correction design, as shown in Fig.~\ref{fig:extension}, stressing the necessity of a client-specific correction factor design. We further conduct extensive experiments on six different datasets in Section \ref{sec:exp} to substantiates the prevalence of over-correction phenomenon.  Based on our literature review, we are not the first work that cast doubt upon the stability and model performance in existing FL methods for non-IID data. Comprehensive studies like~\cite{li2022federated} also report the instability and mediocre performance of these FL algorithms, but they do not provide further investigations and solutions tailored to these phenomena. In this paper, we present theoretical analysis (Section \ref{sec:taco_theory}) and solid empirical evidence (Section \ref{sec:exp}) to unveil the mask of the over-correction phenomenon.

\subsection{Unsatisfactory Time-to-accuracy Performance}
\begin{table}[t]
\centering
\caption{Computation time per 100 local updates (CNN)}
\begin{footnotesize}
\begin{tabular}{cccccc}
\toprule
Dataset                 & \begin{tabular}[c]{@{}c@{}}FedAvg\\ FoolsGold\end{tabular} & FedProx  & Scaffold & STEM     & FedACG   \\ \hline
\multirow{2}{*}{FMNIST} & 0.323s                                                     & 0.399s   & 0.348s   & 0.455s   & 0.401s   \\
                        & +0.0\%                                                     & +23.52\% & +7.73\%  & +40.86\% & +24.15\%  \\ \hline
\multirow{2}{*}{SVHN}   & 0.441s                                                     & 0.511s   & 0.462s   & 0.543s   & 0.513s   \\
                        & +0.0\%                                                     & +15.87\% & +4.76\%  & +23.13\% & +16.33\% \\ \bottomrule
\end{tabular}
\end{footnotesize}
\label{tab:reeval_time}
\vspace{-1.2em}
\end{table}
We notice that STEM appears to show outstanding round-to-accuracy performance in Figs. \ref{fig:reeval-round-fmnist} and \ref{fig:reeval-round-svhn}. However, clients need to perform gradient computations two times in STEM (Line \ref{line:stem_local}, Algorithm \ref{alg:summary}), indicating more computation overhead and thus longer wall-clock time per round. Figs \ref{fig:reeval-time-fmnist} and \ref{fig:reeval-time-svhn} together with Table \ref{tab:reeval_time} show that most algorithms, especially STEM, significantly increase the computation time of local updates in every FL round, which leads to their mediocre time-to-accuracy performance.  We also observe that the incorporation of regularization terms in FedProx and FedACG algorithms can also impose substantial computation burden on gradient computations. Hence, it is imperative to design a client-friendly FL algorithm for non-IID data to avoid imposing extra computational burden on FL clients. 

    In conclusion, we unveil the detrimental effect of over-correction and the poor time-to-accuracy performance of pioneering FL algorithms in this section. Based on these findings, we will propose our algorithm \textbf{TACO} and present theoretical analysis in Section \ref{sec:alg}.

%% file: 4_alg.tex
\section{TACO: FL with Tailored Adaptive Correction}
\label{sec:alg}
In this section, we propose \textbf{TACO}, a lightweight \underline{t}ailored \underline{a}daptive \underline{co}rrection FL algorithm to address the over-correction problem for training with non-IID data, while maintaining great time-to-accuracy performance. We first introduce our algorithm design details and then present a convergence analysis to support our design intuition from the theoretical view.
\subsection{TACO: Algorithm Design}
\label{sec:taco_design}
\begin{figure}[t]
    \centering
    \includegraphics[width=8.5cm]{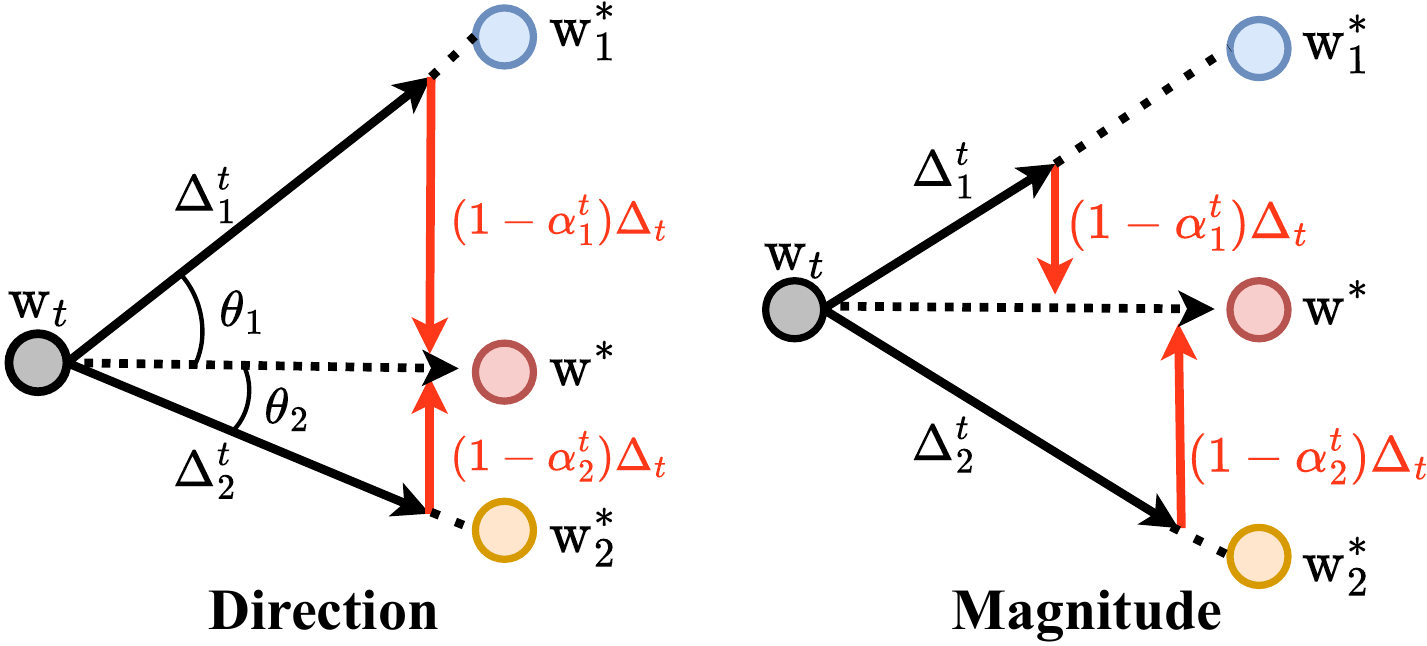}
    \caption{Clients with larger $\theta_i$ and magnitude of the local gradient $\Delta_i^t$ need larger correction factors $1-\alpha_i^t$.}
    \label{fig:corr-cos}
    \vspace{-3mm}
\end{figure}
\noindent\textbf{Tailored, adaptive local model correction.}
The \textbf{`Global guides Local'} principle in previous research inspires us to leverage the global gradient $\Delta_t $ 
to calibrate the local update steps. To further achieve the tailored adaptive correction over different clients, we establish the correction coefficient $\alpha_i^t$ for each client $i$ by utilizing both the {\em magnitude} of the local gradient $(\left\|\Delta_i^t\right\|)$ and {\em direction} based on cosine similarity between $\Delta_i^t$ and the average global gradient $\bar{\Delta}_{t} = \sum_{i\in\mathcal{N}}\Delta_i^{t-1}/N$:
\begin{align}
    \alpha_i^{t}=\left(1 - \frac{||\Delta_i^{t-1}||}{\sum\limits_{i\in\mathcal{N}}||\Delta_i^{t-1}||}\right) \max\left\{\frac{{{\Delta_i^{t-1}}^\top\bar{\Delta}_{t}}} {||\Delta_i^{t-1}||\cdot||\bar{\Delta}_{t}||}, 0\right\}.
    \label{eq:alpha_compute}
\end{align}
The local update step in TACO can be formulated as:
\begin{equation}
\label{eq:local_update_TACO}
    \mathbf{w}_{i,k+1}^t \leftarrow \mathbf{w}_{i,k}^t - \eta_l (g_{i,k}^t + \gamma(1-\alpha_i^t) \Delta_t),
\end{equation}
where $1 - \alpha_i^t$ is the {\bf correction factor} of client $i$ at the communication round $t$, $\gamma\in(0,1)$ is a hyper-parameter to determine the maximum correction factor. $\Delta_t$ is determined by our aggregation scheme and will be formalized in Eq.~\eqref{eq:taco_aggregate}.

We illustrate our design intuition of correction factor $\alpha_i^t$ in Fig.~\ref{fig:corr-cos}. Fig.~\ref{fig:corr-cos}-Left shows that client 1 should have a larger correction factor $1-\alpha_1^t$ for local updates compared to client 2 due to his lower cosine similarity (larger angle $\theta_1$) with the vector pointing towards the global optimum $\mathbf{w^*}$ (the dashed line from $\mathbf{w}_t$ to $\mathbf{w^*}$), so that both clients can appropriately steer their local models towards the global optimum $\mathbf{w}^*$. Meanwhile, Fig.~\ref{fig:corr-cos}-Right presents that client 2 with a larger magnitude of the local updates (larger $||\Delta_i^t||$) should also have a larger correction factor $1-\alpha_2^t$ for local updates compared to client 1. This insight is pivotal to our design in the computation of $\alpha_i^t$ in Eq. (\ref{eq:alpha_compute}) and we will further demonstrate its optimality through theoretical analysis in Section \ref{sec:taco_theory}. We also validate that the computation of $\alpha_i^t$ in TACO incurs limited additional computation overhead compared to the extra computation procedure in prior FL methods, as shown in Figure \ref{fig:compute_time}. 

\noindent\textbf{Tailored aggregation design.} 
The design of the aggregation function for TACO benefits from the design of the correction coefficient $\alpha_i^{t+1}$ in Eq.~\eqref{eq:alpha_compute}. Clients with a higher value of $\alpha_i^{t+1}$ indicate that their local updated gradients $\Delta_i^t$ are more similar to the ideal updated gradients towards the global optimum. Therefore, the values of $\alpha_i^{t+1}$ of different clients can also well reflect the non-IID degrees of their local datasets. To validate this argument, we design a synthetic non-IID distribution by randomly dividing clients into three groups: in \textbf{Group A}, each client has only $10\%$ of the labels; In \textbf{Group B}, each client has about $20\%$ of the labels; In \textbf{Group C}, each client has $50\%$ of the labels. Clients in the same group can have data with different labels which are randomly assigned, and the more labels a client possesses, the more closely its local data distribution aligns with the global data distribution. 
 We compare the average value of the correction coefficients $\alpha_i^t$ across different client groups in Table \ref{tab:average_alpha}. We observe that TACO can distinguish clients with different label diversity and assign higher $\alpha_i^t$'s to those who have more diverse training data. These motivate us to propose the aggregation function of TACO in \eqref{eq:taco_aggregate} by utilizing the correction coefficient $\alpha_i^{t+1}$, and clients with more diverse data can obtain higher aggregation weights during global aggregation, formalized as follows.
\begin{equation}
\label{eq:taco_aggregate}
    \Delta_{t+1} = \frac{1}{K\eta_l\sum\limits_{i\in\mathcal{N}}\alpha_i^{t+1}} \sum_{i=1}^{N}\alpha_i^{t+1}\Delta_i^t.
\end{equation}

\setlength{\tabcolsep}{3pt}
\begin{table}[t]
\centering
\caption{Average value of $\alpha_i^t$ of different groups of clients.} 
\begin{tabular}{ccccc}
\toprule
\textbf{Group}      & \textbf{MNIST} &\textbf{FMNIST} &\textbf{SVHN} &\textbf{CIFAR10}  \\ \hline
Group A                                 & 0.23±0.02          & 0.18±0.09         & 0.23±0.09        & 0.22±0.07 \\ 
Group B                              & 0.27±0.03          & 0.27±0.03          & 0.33±0.01        & 0.32±0.01 \\ 
Group C                              & 0.41±0.02          & 0.37±0.01          & 0.43±0.01        & 0.36±0.01 \\ 
\textbf{Freeloaders}                           & \textbf{0.88±0.01}          & \textbf{0.75±0.01}          &  \textbf{0.86±0.01}       & \textbf{0.88±0.06} \\ \bottomrule
\end{tabular}
\label{tab:average_alpha}
\vspace{-1.6em}
\end{table}
\noindent\textbf{Freeloader detection.} 
We further demonstrate the ingenuity of our design of $\alpha_i^t$ in \eqref{eq:alpha_compute}, which not only assists tailored corrections and aggregation, but also provides a bonus of making TACO less dependent on authenticity. 
Freeloaders~\cite{fraboni2021free, zhang2022enabling,lin2019free,chen2024toward} refer to lazy clients that only upload previous global gradients $\Delta_t$ received without contributing any new local updates. To discourage this behavior and ensure system stability, we propose a lightweight inspection procedure (Line \ref{line:inspect_lazy}, Algorithm \ref{alg:taco}) in the global aggregation step, removing the incentives of being freeloaders. In table \ref{tab:average_alpha}, we discover that the global gradients between two consecutive aggregation rounds $\Delta_t$ and $\Delta_{t+1}$ have a relatively higher cosine similarity (the second term of $\alpha_i^{t+1}$ in (\ref{eq:alpha_compute})), leading to significantly higher values of $\alpha_i^t$ for freeloaders compared to normal clients. Therefore, we can filter out clients with significantly high values of $\alpha_i^{t+1}$ as suspicious clients, and we denote $\kappa$ as the detection threshold:
\begin{equation}
    \alpha_i^{t+1} \geq \kappa, i\in \mathcal{N}.
    \label{eq:lazy-threshold}
\end{equation}
TACO will expel clients satisfying Eq.~\eqref{eq:lazy-threshold} over $\lambda$ times. We evaluate TACO's ability to distinguish between freeloaders and regular clients in non-IID FL training (Section \ref{sec:exp}, Table \ref{tab:free_thres}). 

\begin{algorithm}[t] \SetKwData{Left}{left}\SetKwData{This}{this}\SetKwData{Up}{up} \SetKwFunction{Union}{Union}\SetKwFunction{FindCompress}{FindCompress} \SetKwInOut{Input}{Input}\SetKwInOut{Output}{Output}
	\Input{Step size $\eta_l, \eta_g$, batch size $s$, number of local updates $K$, number of communication rounds $T$, maximum correction $\gamma$, thresholds $\kappa, \lambda$}
	\Output{$\mathbf{z}_T$}
	\SetKwInOut{Initialize}{Initialize}
	\Initialize{$\mathbf{w}_0, \Delta_0 \leftarrow \mathbf{0}$, $\alpha_i^0 \leftarrow 0.1, \forall i$}
	\For{$t=0, 1, ..., T-1$}{
	    	\For{each client $i \in \mathcal{N}$ \bf{in parallel}}{
                    $\mathbf{w}_{i,0}^t \leftarrow \mathbf{w}_t$ \;
                    \For{$k=0,1,...,K-1$}{
                        Compute gradient $g_{i,k}^t= \nabla f_i(\mathbf{w}_{i,k}^t, \xi_{i,k}^t)$\;
                        Perform local update according to Eq.~(\ref{eq:local_update_TACO})\;
                    }
                $\Delta_i^t \leftarrow \mathbf{w}_{i,0}^t - \mathbf{w}_{i,K}^t$\tcp*{Local gradient}
                Send the local gradient $\Delta_i^t$ to server\;
		}

            
	\tcp{Server side}
            Compute the $\alpha_i^{t+1}$ according to \eqref{eq:alpha_compute} \label{line:alpha}\;
            
            Compute the global gradient $\Delta_{t+1}$ using Eq. (\ref{eq:taco_aggregate}) \label{line:delta_t2}\;
		
		$\mathbf{w}_{t+1} = \mathbf{w}_t - \eta_g\Delta_{t+1}$\;
            Expel clients if satisfying Eq.~(\ref{eq:lazy-threshold}) $\lambda$ times \label{line:inspect_lazy}\;
            Send $\mathbf{w}_{t+1}, \alpha_i^{t+1}$ to each client $i$
	}
        Compute $\mathbf{z}_T$ based on Eq.~(\ref{eq:zt})
	\caption{TACO (Tailored Adaptive Correction)}
 \label{alg:taco}
\end{algorithm}
\setlength{\tabcolsep}{5pt}
\begin{table}[t]
\centering
\begin{footnotesize}
\caption{Comparison with pioneering FL algorithms}
\begin{tabular}{ccccc}
\toprule
              & \begin{tabular}[c]{@{}c@{}}Local\\ Corr.\end{tabular} & \begin{tabular}[c]{@{}c@{}}Agg.\\ Corr.\end{tabular} & \begin{tabular}[c]{@{}c@{}}Freeloader\\ Detection\end{tabular} & \begin{tabular}[c]{@{}c@{}}Client Overhead\\(Time per round)\end{tabular} \\ \hline
FedAvg        & \ding{55}                                                    & \ding{55}                                                          & \ding{55}                                                   & Low\quad\;\; (4.50±0.02s)    \\ 
FedProx       & \ding{51}                                                    & \ding{55}                                                          & \ding{55}                                                   & Medium (5.05±0.02s)    \\ 
Scaffold      & \ding{51}                                                    & \ding{55}                                                          & \ding{55}                                                   & Medium (5.01±0.02s)   \\ 
FoolsGold     & \ding{55}                                                    & \ding{51}                                                          & \ding{55}                                                   & Low \quad\;\;(4.50±0.02s)   \\ 
STEM          & \ding{51}                                                    & \ding{51}                                                          & \ding{55}                                                   & High \quad\;\,(6.48±0.01s)   \\ 
FedACG        & \ding{51}                                                    & \ding{51}                                                          & \ding{55}                                                   & Medium (5.07±0.02s)   \\ 
\textbf{TACO} & \ding{51}                                                    & \ding{51}                                                          & \ding{51}                                                   & Low \quad\;\;(4.81±0.01s)   \\ \bottomrule
\end{tabular}
\vspace{-4.2mm}
\label{tab:brief_compare}
\end{footnotesize}
\end{table}
To conclude, we present Table \ref{tab:brief_compare} to summarize the differences between \textbf{TACO} and other pioneering algorithms 
across four dimensions: Local update correction, global aggregation correction, freeloader detection, and computation overhead of clients (local computation time per round training ResNet18 on the CIFAR-100). Table \ref{tab:brief_compare} shows that other algorithms cannot achieve satisfactory performance across all four dimensions simultaneously, unable to balance model corrections, training efficiency, and robustness. Among these algorithms, FoolsGold and FedACG come closest to TACO. The primary advantage of TACO is its capability to handle the over-correction problem induced by non-IID data and identify freeloaders without imposing high computation burden on clients. 

\vspace{-3.2mm}
\subsection{TACO: Convergence Analysis}
\label{sec:taco_theory}
In this section, we derive the error bound (Theorem \ref{thm:1}) and convergence analysis (Corollary \ref{cor:converge_rate}) of FL with our tailored correction coefficients $\alpha_i^t$ to quantify the negative effect of model over-correction theoretically, based on which we present a corollary to demonstrate the superiority of TACO's design for $\alpha_i^t$ (Corollary \ref{cor:optimal-alpha}). We begin by listing assumptions 
commonly adopted in pioneering FL research~\cite{xu2021fedcm,bottou2018optimization, wang2019adaptive, rodio2023federated, wang2023federated,liu2023dynamite}.
\begin{asmp}
\label{asmp1}
($L$-smooth) For each client $i\in \mathcal{N}$,  and some scalar $L > 0$, each local loss function $f_{i}$ satisfies: $\left \| \nabla f_{i}(\mathbf{w}_1)-\nabla f_{i}(\mathbf{w}_2)  \right \| \leq L\left \| \mathbf{w}_1-\mathbf{w}_2  \right \| \text{\emph{ for all }} \mathbf{w}_1, \mathbf{w}_2$. 
\end{asmp}
\begin{asmp}
\label{asmp2}
(Heterogeneous and bounded cosine similarity between the local gradients and true global gradient) For some scalars $\mu_i>0, c_i > 0$, the accumulated local gradient $\Delta_i^t$ of client $i$ and the true global gradient $\nabla f(\mathbf{w}_t)$ in round $t$, satisfy: 
\begin{align}
		(\nabla f(\mathbf{w}_t))^\top\mathbb{E}[\Delta_i^{t}]&\leq\,\mu_i\left\|\nabla f(\mathbf{w}_t)\right\|^2 \label{eq:asmp2-1},\\
            \cos(\nabla f(\mathbf{w}_t), \mathbb{E}[\Delta_i^{t}]) &\geq  c_i.\label{eq:asmp2-2}
\end{align}
\end{asmp}
\begin{asmp}
\label{asmp3}
(Bounded gradient) For some scalar $G > 0$, the norm of the true global gradient $\nabla f(\mathbf{w})$ satisfies: $\|\nabla f(\mathbf{w})\| \leq G, \forall \mathbf{w}$.
\end{asmp}



\begin{rem}
    Prior research~\cite{kairouz2021advances} summarizes various assumptions to describe the non-IID property of clients' local training data in existing FL works, \textit{e.g.}, bounded inter-client gradient variance~\cite{wang2023fedmos}, bounded optimal objective difference~\cite{li2019convergence}, and bounded gradient dissimilarity~\cite{li2020federated}. However, they all impose homogeneous bounds across clients, providing only an overall description of the non-IID training data, rather than a detailed and client-specific representation for each client. Our Assumption \ref{asmp2} represents the first attempt to capture diverse non-IID degrees among client data. This characterization $(\mu_i, c_i)$ is crucial 
    for quantifying the deviation between each client's local optimum ($\mathbf{w}_i^*$) and the global optimum $(\mathbf{w}^*)$, providing a design knob for our client-specific corrections. Moreover, Assumption~\ref{asmp2} is weaker than the common requirement for unbiased mini-batch gradients.
\end{rem}
We further introduce definitions and associated lemmas crucial for proving the main theorem. The proof of Theorem \ref{thm:1} and Corollary \ref{cor:converge_rate}-\ref{cor:optimal-alpha} are presented in the Appendix.
\begin{defi}
We define $\Tilde{\Delta}_t$ as the average mini-batch gradient over clients in communication round $t$, formalized as follows.
    \begin{equation}
    \label{eq:tilde_delta_t}
        \Tilde{\Delta}_t = \frac{1}{ KN}\sum_{i\in\mathcal{N}}\sum_{k=0}^{K-1}g_{i,k}^{t}. 
    \end{equation}
\end{defi}
\begin{defi}
    We define $\alpha_t$ as the average correction coefficient applied over all clients in communication round $t$ and $\mathbf{z}_t$ represents the final model output after communication round $t$, formalized as follows.
    \begin{align}
        \alpha_t &= 1 - \frac{1}{N}\sum_{i\in\mathcal{N}}(1-\alpha_i^t) = \sum_{i\in\mathcal{N}}\alpha_i^t/N, \label{eq:alpha_t}\\
        \mathbf{z}_t &= \mathbf{w}_t + (1-\alpha_t)(\mathbf{w}_t - \mathbf{w}_{t-1}), \label{eq:zt}
    \end{align}
    where $\mathbf{w}_t$ is the aggregated model parameter in communication round $t$, and we have $\mathbf{z}_t =\mathbf{w}_t$ if $\alpha_i^t = 1, \forall i \in \mathcal{N}$. 
\end{defi}
\begin{lem}[Update rule of the aggregated global gradient $\Delta_t$]
	$\Delta_t$ can be treated as the exponential moving average of the accumulated local gradients from all clients.
	\begin{equation}
 		\Delta_{t+1} = \Tilde{\Delta}_{t} + (1-\alpha_t)\Delta_t.
	\end{equation}
    \label{lem:Delta_t}
\end{lem}
\vspace{-5mm}
\begin{lem}[Update rule of the final model output {$\mathbf{z}_t$}]
	\begin{equation}
 		\mathbf{z}_{t+1} = \mathbf{z}_t - \eta_g\Tilde{\Delta}_t.
	\end{equation}
 \label{lem:z_t}
\end{lem}
\vspace{-5mm}




\vspace{-1em}
\begin{proof}[Proof Sketch]
    Lemma \ref{lem:Delta_t} can be derived by the definition of $\Delta_t$ in Eq.~\eqref{eq:taco_aggregate} and local update rules in Eq.~\eqref{eq:local_update_TACO}. Lemma \ref{lem:z_t} can be derived by the definition of $\mathbf{z}_t$ in Eq.~\eqref{eq:zt} and Lemma \ref{lem:Delta_t}. 
\end{proof}
\begin{thm}[Error bound with over-correction term and tailored correction coefficients $\alpha_i^t$]
\label{thm:1}
	Suppose that the loss functions satisfy Assumptions \ref{asmp1}--\ref{asmp3}. Given fixed learning rates $\eta_g$ and $\eta_l$, the expected global loss 
 after $t+1$ aggregation rounds with $K$ local updates per round $\mathbb{E}\left[f(\mathbf{z}_{t+1})\right]$ is at most:
\begin{equation*}
\mathbb{E}\left[f(\mathbf{z}_{t})\right]-\frac{\eta_{g}}{2}\mathbb{E}\left[\|\nabla f(\mathbf{z}_{t})\|^{2}\right]+\frac{L}{2}\eta_{g}^{2}\mathbb{E}\left[\|\tilde{\Delta}_{t}\|^{2}\right]+\eta_{g}L^{2}\varepsilon_{t}+ \eta_g^3Y_t, 
\end{equation*}
where $Y_t = \frac{L^2G^2}{K^2N^4\eta_l^2}\left(\sum_{i\in\mathcal{N}}(1-\alpha_i^t)\sum_{i\in\mathcal{N}}\frac{\mu_i}{c_i} \right)^{2}$ is the term caused by the correction coefficient $\alpha_i^t$ of each client $i$, and $\varepsilon_{t}=\frac1{KN}\sum_{i\in[N],k\in[k]}\mathbb{E}\left[\|\mathbf{w}_{t}-\mathbf{w}_{i,k}^{t}\|^{2}\right]$. 
\end{thm}
\begin{cor}[Convergence rate with tailored correction coefficients $\alpha_i^t$]
Suppose that the loss functions satisfy Assumptions \ref{asmp1}--\ref{asmp3}, the convergence rate of FL with tailored correction coefficients $\alpha_i^t$ under non-convex loss function assumptions is:
\begin{equation}
    \frac{1}{T}\sum_{t=1}^{T}\mathbb{E}\left[\|\nabla f(\mathbf{z}_{t})\|^{2}\right] = \mathcal{O}\left(\sqrt{\frac{L}{T}}+\sqrt[3]{\frac{Y}{T^2}}\right).
    \label{eq:convergence_rate}
\end{equation}
where $Y = \mathop{\max}_{t\in[T]} Y_t$ is the term caused by the correction coefficient $\alpha_i^t$ of every client $i$. Especially, when $\alpha_i^t=\alpha_{i'}^t,~\forall i\neq i'$, Corollary \ref{cor:converge_rate} is consistent with the convergence rate in prior work~\cite{xu2021fedcm}.
\label{cor:converge_rate}
\end{cor}
By explicitly presenting the terms determined by correction coefficients, \textit{i.e.} $Y_t$ and $Y$, Theorem \ref{thm:1} and Corollary \ref{cor:converge_rate} provide the first FL error bound and non-convex convergence rate that unveils the prevalence of the over-correction in FL as well as its detrimental effect on the error bound and convergence rate, respectively. Existing methods that utilize uniform correction coefficients $\alpha_i^t$, can easily cause a larger value of $Y_t$ and $Y$ and thus lead to suboptimal model performance. 
Therefore, a tailored correction coefficients design is essential to minimize $Y_t$ and $Y$ to alleviate the negative effect of over-correction. To achieve this, we further present a 
corollary based on Theorem \ref{thm:1} to guide our adaptive adjustment of the correction coefficient $\alpha_i^t$ for each client $i$ in every communication round $t$.

\begin{cor}[Optimal correction factors for different clients]
\label{cor:optimal-alpha}
For some scalar $\sigma > 0$, if $ \sum_{i\in\mathcal{N}}(1-\alpha_i^t)\geq \sigma$, to minimize the $Y_t$ in error bound and $Y$ in convergence rate, the correction coefficient $\alpha_i^t$ of client $i$ in training round $t$ satisfies:
\begin{equation}
    (1-\alpha_i^t)\propto \frac{\mu_i}{c_i},
\end{equation}
where $c_i$ reflects the cosine similarity between the local gradient of client $i$ and the global gradient, and $\mu_i$ 
reflects the magnitude of the local gradient $\mathbb{E}\left[\left\|\Delta_i^t\right\|\right]$ (Assumption \ref{asmp2}).
\end{cor}

Corollary \ref{cor:optimal-alpha} supports TACO's design in the correction coefficient $\alpha_i^t$ of client $i$ in (\ref{eq:alpha_compute}) from a theoretical perspective, which is illustrated in Fig. \ref{fig:corr-cos}. Clients with a larger magnitude of the local updates (larger $\mu_i$) or a lower degree of cosine similarity (smaller $c_i$) with the global gradient should have a smaller correction coefficient $\alpha_i^t$ in (\ref{eq:alpha_compute}) and thus larger correction factor $1-\alpha_i^t$ in \eqref{eq:local_update_TACO} to calibrate their local updates. 

%% file: 5_experiment.tex
\section{Experimental Results}
\label{sec:exp}
\subsection{Experiment Setup}
In this section, we assess our TACO against pioneering algorithms for FL training with non-IID data to show its superiority. All experiments are conducted on the server with 12th Gen Intel(R) Core(TM) i9-12900K and NVIDIA GeForce 3090 GPU with 20GB CUDA memory.

\noindent\textbf{Baselines.}
We compare TACO with FedAvg~\cite{mcmahan2017communication}, FedProx~\cite{li2020federated}, FoolsGold~(FG)\cite{fung2020limitations}, Scaffold~\cite{karimireddy2020scaffold}, STEM~\cite{khanduri2021stem} and the latest SOTA algorithm FedACG~\cite{fedacg} using 20 clients with full participation. They represent different technical approaches to addressing the non-IID challenge, as illustrated in Section \ref{sec:reeval}.

\noindent\textbf{Datasets, models, and non-IID settings.} For model architectures, we employ an MLP model with three hidden layers (32, 16, 8) to train on a tabular dataset (adult~\cite{misc_adult_2}). For most image datasets, we adopt a CNN model~\cite{li2022federated} with two 5 × 5 convolutional layers and three fully connected layers with ReLu activation, and we adopt ResNet18~\cite{he2016deep} for the complex dataset CIFAR-100. We adopt LSTM model~\cite{hu2022autofl} for the non-IID word prediction dataset Shakespeare introduced in LEAF~\cite{caldas2018leaf} (a classic FL benchmark), which are all implemented by Pytorch. For non-IID distribution, we focus on the label distribution skew, which is the most challenging non-IID data setting compared to others~\cite{li2022federated}. We adopt the Dirichlet distribution (Dir($\phi$)) with different concentration parameters $\phi$ to simulate various label imbalanced scenarios, which is widely used in a variety of FL studies~\cite{deng2023hierarchical,shi2023distributionally,jiang2023heterogeneity,chaninternal,liu2024can}. We also design a \textbf{synthetic} label distribution comprising three groups of clients with distinct label diversity, as shown in table \ref{tab:average_alpha} in Section \ref{sec:taco_design}, to better simulate the non-IID data distribution in practical FL settings. The detailed information regarding the datasets, models, and non-IID distributions of the experiments has presented in Table \ref{tab:datasets}, unless otherwise specified.
\setlength{\tabcolsep}{3pt}
\begin{table}[t]
\caption{Datasets, models, and non-IID distributions}
\centering
\footnotesize
\begin{tabular}{cccccc}
\toprule
\textbf{Datasets} & \textbf{\begin{tabular}[c]{@{}c@{}}\# of training\\ samples\end{tabular}} & \textbf{\begin{tabular}[c]{@{}c@{}}\# of test\\ samples\end{tabular}} & \textbf{\# of classes} & \textbf{Non-IID} & \textbf{Model} \\ \hline
MNIST            & 60000     & 10000                                                            & 10                 & Synthetic        & CNN    \\
FMNIST            & 60000     & 10000                                                            & 10                 & Synthetic        & CNN            \\ 
FEMNIST            & 341873    & 40832    & 62                 & Dir(0.2)        & CNN            \\ 
SVHN              & 73257                                                                & 26032                                                            & 10                 & Synthetic        & CNN            \\ 
CIFAR10          & 50000                                                                & 10000                                                            & 10                 & Synthetic        & CNN            \\ 
CIFAR100         & 50000       & 10000     & 100           & Dir(0.5)        & ResNet18                 \\ 
adult             & 32561    & 16281     & 2                  & Dir(0.5)        & MLP \\
Shakespeare      & 448340        & 70657  & -                 & LEAF~\cite{caldas2018leaf}        & LSTM            \\ 
\bottomrule
\end{tabular}
\label{tab:datasets}
\vspace{-7mm}
\end{table}

\noindent\textbf{Evaluation metrics.} To evaluate the \textit{round-to-accuracy} performance of all these FL methods, we compare: (i) their model performance after a fixed number of communication rounds and (ii) the number of communication rounds required for each algorithm to reach the same target accuracy. To compare their \textit{time-to-accuracy} performance, we measure: (i) the total computation time required by clients to attain a desired model accuracy and (ii) the computation time required by clients to complete the same number of local updates in each communication round. In the \textit{time-to-accuracy} evaluation, to ensure fair comparison, we assume all FL algorithms are implemented in identical network conditions, and thus we do not take account of the network communication time and compare the local computation time in the \textit{time-to-accuracy} evaluation. In certain scenarios, network transmission time might dominate the total training time. In this case, the number of training (communication) rounds directly determines the overall training time, and the \textit{round-to-accuracy} metric can reflect the training efficiency of different FL algorithms. Therefore, the evaluation of both \textit{round-to-accuracy} and \textit{time-to-accuracy} metrics enables a more comprehensive comparison of training efficiency across different FL algorithms.

To verify TACO's ability to identify freeloaders, we use true positive rate $= \frac{\text{\# of identified freeloaders}}{\text{\# of freeloaders}}$ (\textbf{TPR}) and false positive rate $= \frac{\text{\# of misjudged clients}}{\text{\# of normal clients}}$ (\textbf{FPR}) as evaluation metrics, where a larger value of TPR and a smaller value of FPR represent a better performance. We replace $40\%$ of the clients (8 out of 20) with freeloaders in the experiments shown in Table \ref{tab:average_alpha} and \ref{tab:free_thres}, while all clients are benign in the other experiments.

\noindent\textbf{Hyper-parameter configurations.} The default configuration of all baselines and TACO is as follows: The numbers of communication rounds are $T=50$ for the adult and Shakespeare datasets, $T=100$ for MNIST, FMNIST, FEMNIST, and SVHN datasets, and $T=200$ for the CIFAR-10 and CIFAR-100 dataset. The numbers of local update steps are $K = 100$ for MNIST, FMNIST, FEMNIST, and adult datasets, $K=200$ for CIFAR-100 and Shakespeare, and $K=1000$ for SVHN and CIFAR-10 datasets. We set the mini-batch size as $s=64$ and learning rates as $\eta_l = 0.01$ ($\eta_l=1$ for the Shakespeare dataset), $\eta_g = K\eta_l$ unless other specified, which are consistent with the settings in existing FL research~\cite{fedacg, hu2022autofl}. The unique hyper-parameters of each method are: $\zeta = 0.1$ in FedProx, $\alpha = 1$ in Scaffold, $\alpha_t = 0.2$ in STEM, $\beta = 0.001$ in FedACG, and $\lambda = T/5, \kappa = 0.6, \gamma = \frac{1}{K}$ in TACO.
\setlength{\tabcolsep}{4pt}
\begin{table*}[t]
\caption{Round-to-Accuracy performance of various algorithms across different datasets}
\centering
\begin{footnotesize}
\begin{tabular}{ccccccccccccc}
\hline
Datasets          & \multicolumn{2}{c}{adult}   &  \multicolumn{2}{c}{FMNIST}                                           & \multicolumn{2}{c}{SVHN}                                             & \multicolumn{2}{c}{CIFAR10}                                          & \multicolumn{2}{c}{CIFAR100}  & \multicolumn{2}{c}{Shakespeare}                                                                                   \\ \hline
Methods   & \begin{tabular}[c]{@{}c@{}}Acc(\%)\\ 50R \end{tabular}  & \begin{tabular}[c]{@{}c@{}}Rounds\\ (78\%)\end{tabular}  & \begin{tabular}[c]{@{}c@{}}Acc(\%)\\ 100R \end{tabular}    & \begin{tabular}[c]{@{}c@{}}Rounds\\ (70\%)\end{tabular} & \begin{tabular}[c]{@{}c@{}}Acc(\%)\\ 100R \end{tabular}    & \begin{tabular}[c]{@{}c@{}}Rounds\\ (70\%)\end{tabular} & \begin{tabular}[c]{@{}c@{}}Acc(\%)\\ 200R \end{tabular}    & \begin{tabular}[c]{@{}c@{}}Rounds\\ (50\%)\end{tabular} & \begin{tabular}[c]{@{}c@{}}Acc(\%)\\ 200R \end{tabular}    & \begin{tabular}[c]{@{}c@{}}Rounds\\ (54\%)\end{tabular} & \begin{tabular}[c]{@{}c@{}}Acc(\%)\\ 50R \end{tabular}   & \begin{tabular}[c]{@{}c@{}}Rounds\\ (50\%)\end{tabular} \\ \hline
FedAvg    & 78.88±0.06 & 24                                                        & 71.36±0.06 & 85                                                      & 73.97±2.46 & 37                                                      & 51.84±0.76 & 120                                                      & 54.39±0.08 & 198  & 50.54±0.08 & 45                                                                                                         \\
FedProx   & 78.92±0.03 & 24                                                       & 67.14±0.11 & 100+                                                      & 36.69±0.00 & $\times$                                                      & 49.95±0.39 & 200+                                                      & 53.32±0.13 & 200+  & 28.07±0.19 & $\times$                                                                                                           \\
FoolsGold & 83.50±0.16 & 20                                                        & 73.09±0.16 & 80                                                      & 76.01±0.09 & 21                                                      & 53.70±0.51 & 58                                                      & 54.19±0.52 & 200  & 50.49±0.13 & 45                                                                                                        \\
Scaffold  & 82.24±0.90 & 10                                                        & 68.84±0.14 & 100+                                                      & 19.58±0.00 & $\times$                                                      & 43.18±2.69 & $\times$                                                      & 53.33±0.17 & 200+  & 51.95±0.09 & 37                                                                                                           \\
STEM      & 78.83±0.09 & 22                                                       & 72.97±0.03 & 69                                                      & 74.55±2.15 & \textbf{16}                                                      & 52.01±1.30 & 112                                                      & 54.69±0.37 & 191   & 52.57±0.12 & 34                                                                                                        \\
FedACG    & 83.69±0.08 & 11                                                        & 70.91±0.27 & 92                                                      & 75.51±0.17 & 36                                                      & 52.81±0.64 & 116                                                      & 54.07±0.07 & 197   & 42.91±0.05 & $\times$                                                                                                        \\
\textbf{TACO}  & \textbf{83.80±0.06} & \textbf{8}                                                          & \textbf{73.28±0.16} & \textbf{59}                                                      & \textbf{78.26±0.85} & 17                                                      & \textbf{54.98±0.24} & \textbf{44}                                                      & \textbf{56.08±0.26} & \textbf{170}  & \textbf{53.55±0.06} & \textbf{22}                                                                                                      \\ \hline
\end{tabular}
\end{footnotesize}
\label{tab:all_compare}
\vspace{-0.5em}
\end{table*}
\begin{figure*}[t]
    \centering
    \includegraphics[width=0.95\linewidth]{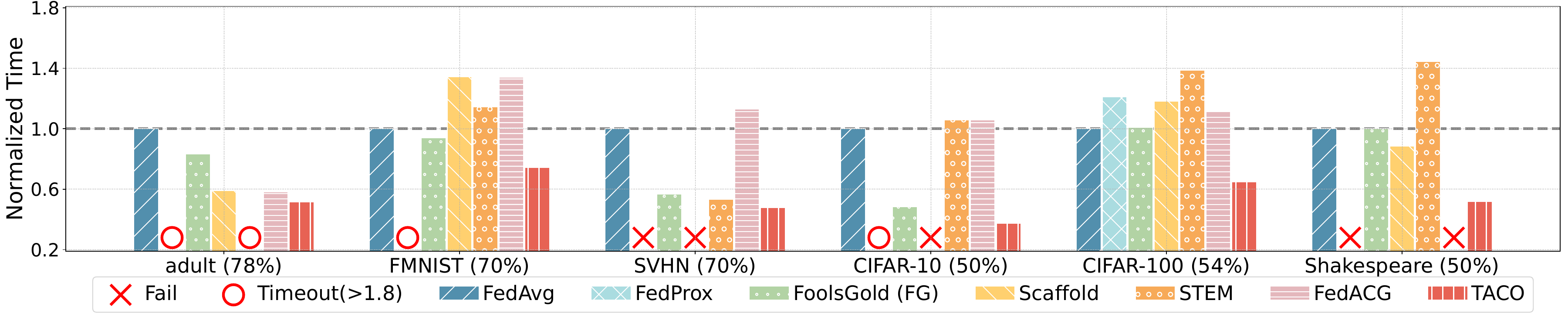}
    \caption{Cumulative local training time required by different algorithms to achieve the target accuracy.}
    \label{fig:time-to-acc}
    \vspace{-1.5em}
\end{figure*}

\begin{figure}
	\centering
	\begin{subfigure}[t]{0.49\linewidth}
		\includegraphics[width=\textwidth]{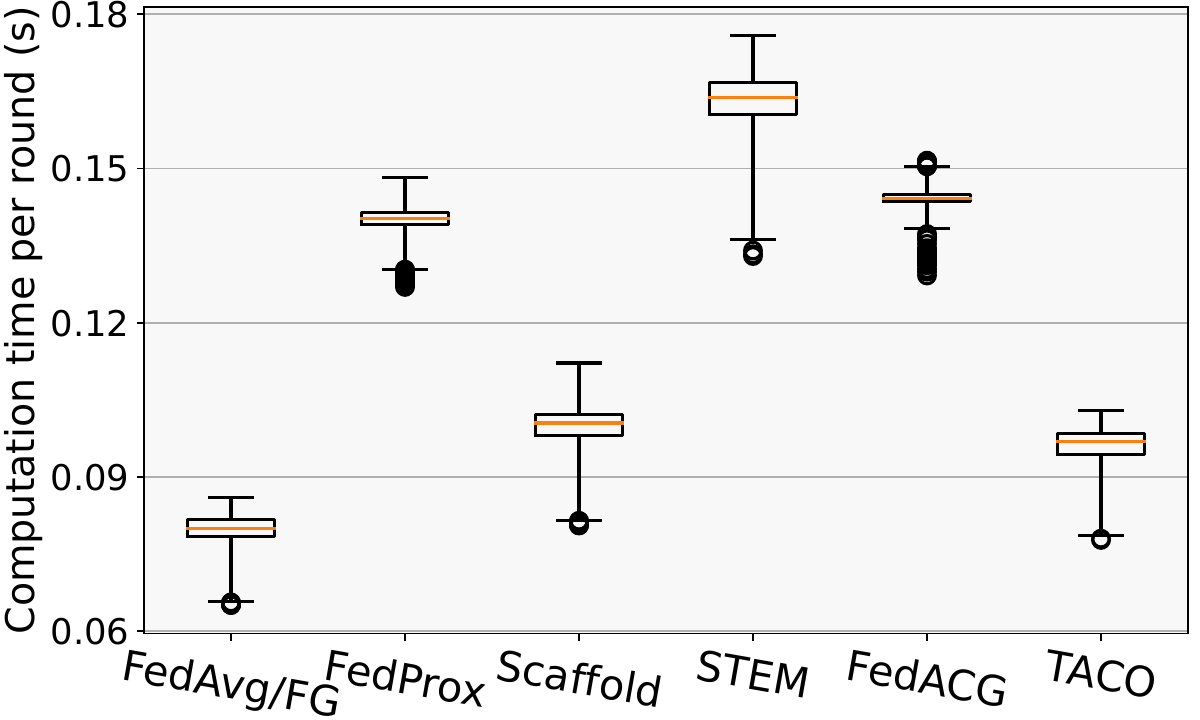}
		\caption{adult-MLP}
	\end{subfigure}
	\begin{subfigure}[t]{0.49\linewidth}
		\includegraphics[width=\textwidth]{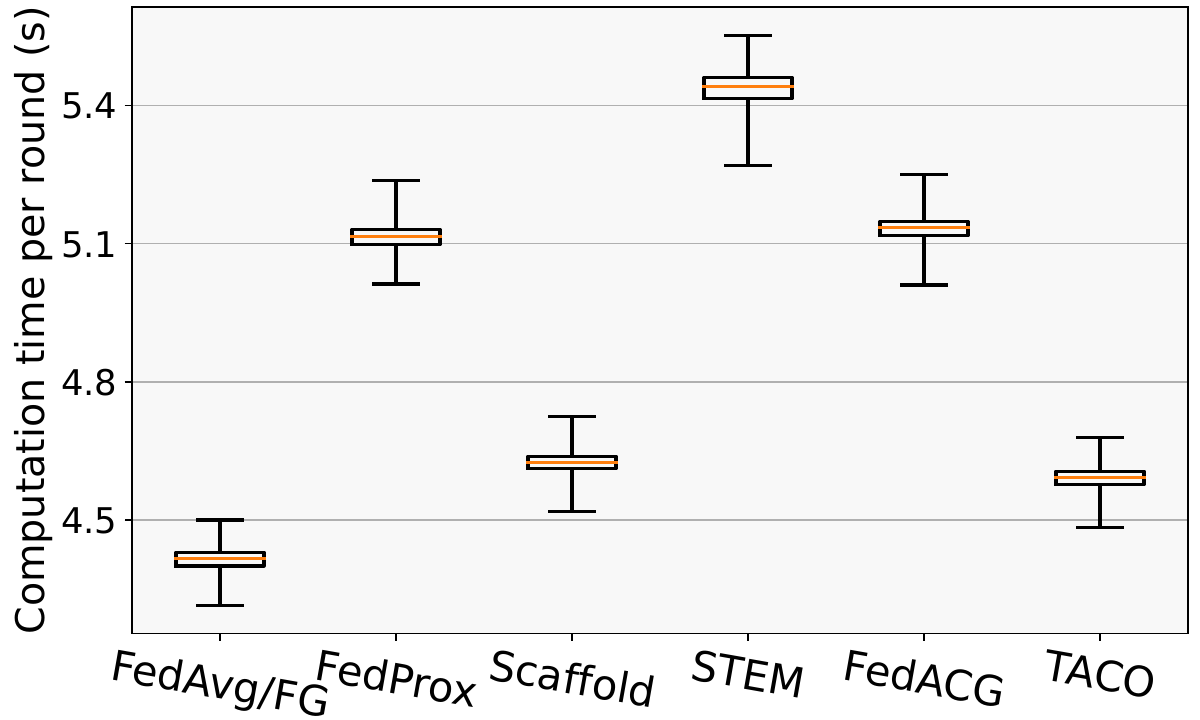}
		\caption{SVHN-CNN}
	\end{subfigure}
        \begin{subfigure}[t]{0.49\linewidth}
		\includegraphics[width=\textwidth]{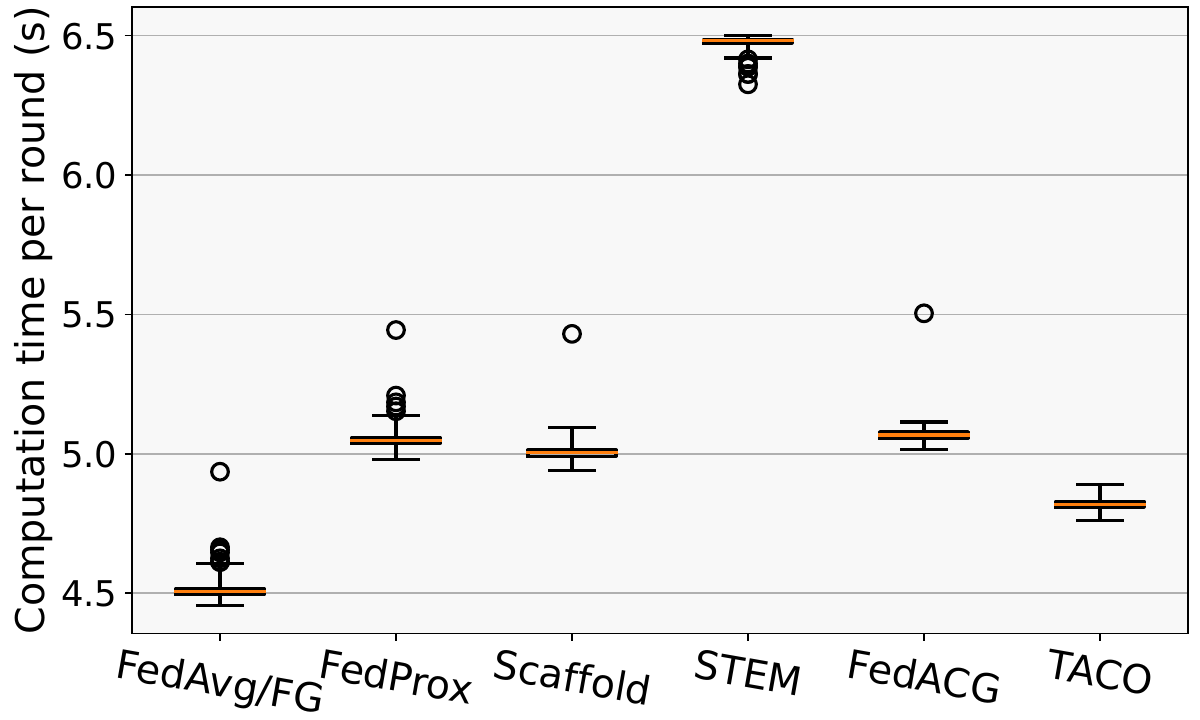}
		\caption{CIFAR100-ResNet}
	\end{subfigure}
        \begin{subfigure}[t]{0.49\linewidth}
		\includegraphics[width=\textwidth]{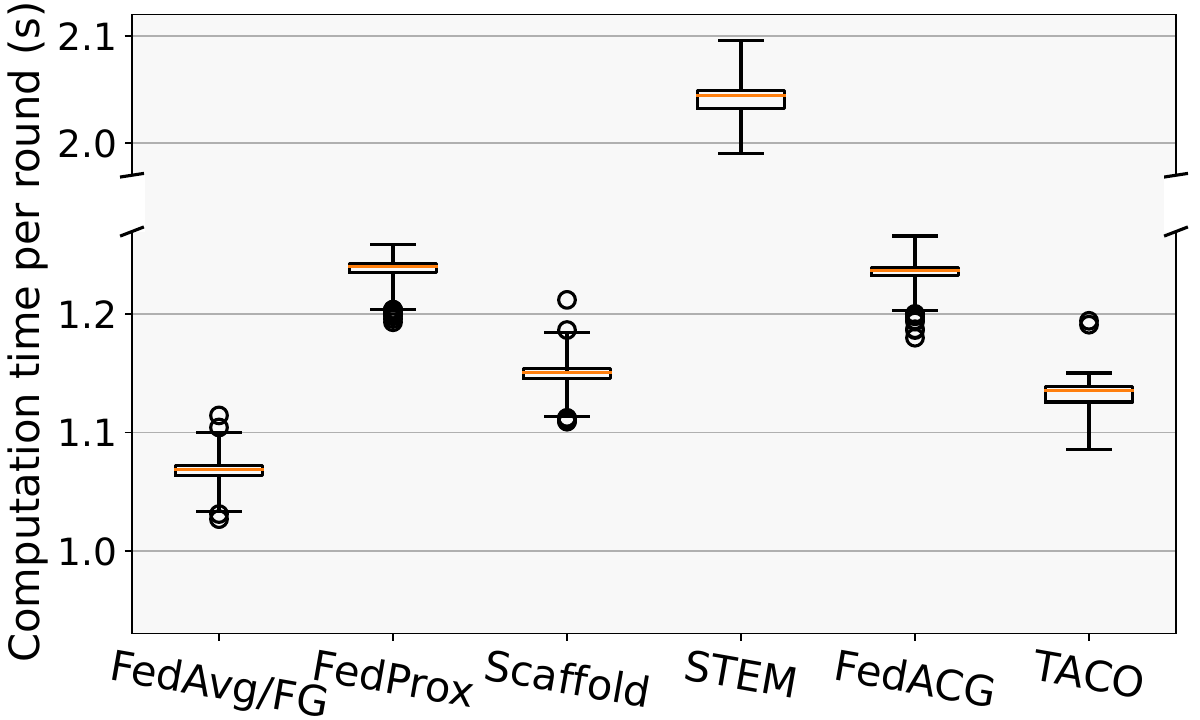}
		\caption{Shakespeare-LSTM}
	\end{subfigure}
	\caption{Local computation time for clients in every FL round. (The orange bars are the median across all training rounds.)}
	\label{fig:compute_time}
\end{figure}
\subsection{Round-to-accuracy Performance}

Table \ref{tab:all_compare} shows that TACO outperforms the baselines in six different datasets after performing the same number of training rounds. {\bf TACO} can obtain up to 2.76\%--58.68\% final test accuracy improvement across diverse datasets and non-IID settings, indicating the superiority of our tailored correction design. It also shows that other FL methods have little improvement or even worse model performance compared to FedAvg due to the over-correction. \textbf{FedProx} and \textbf{Scaffold} even fail to achieve model convergence on the SVHN dataset, supporting Section \ref{sec:reeval}. We further conduct experiments, where we refine FedProx and Scaffold by replacing their coefficients $\zeta$ and $\alpha$ with our tailored correction coefficients $\alpha_i^t$. 
Fig. \ref{fig:extension} shows performance improvements of such integration, highlighting a broader perspective that client-specific corrections are of great significance in non-IID data training. 

\subsection{Time-to-accuracy Performance}
We record the actual local computation time of the slowest client in every FL round in Fig.~\ref{fig:compute_time}. We then sum them up to compare the total client computation time of different algorithms to reach the target accuracy in Fig.~\ref{fig:time-to-acc}.
Fig.~\ref{fig:compute_time} shows that most FL methods, except FoolsGold, suffer from a significantly longer computation time due to their required extra computation by clients. This explains why algorithms like STEM can often achieve better round-to-accuracy performance but worse time-to-accuracy performance than FedAvg, as shown in Fig. \ref{fig:time-to-acc}. STEM algorithm can require up to over 80\% more local computation time (defined as Timeout) on the client side than FedAvg to meet the same model performance, despite a fewer number of training rounds needed. Besides, other baseline methods like FedProx and Scaffold sometimes require more training rounds than FedAvg or even fail to reach the target accuracy due to over-correction, which can also lead to convergence failure ($\times$) or timeout ($\circ$), as shown in Table \ref{tab:all_compare} and Fig.~\ref{fig:time-to-acc}. In contrast, TACO incurs minimal additional computation overhead and reduces 25.6\%--62.7\% client computation time to reach the target accuracy compared to FedAvg, making it the most time-efficient one among all the algorithms.
\begin{figure}
	\centering
	\begin{subfigure}[t]{0.48\linewidth}
		\includegraphics[width=\textwidth]{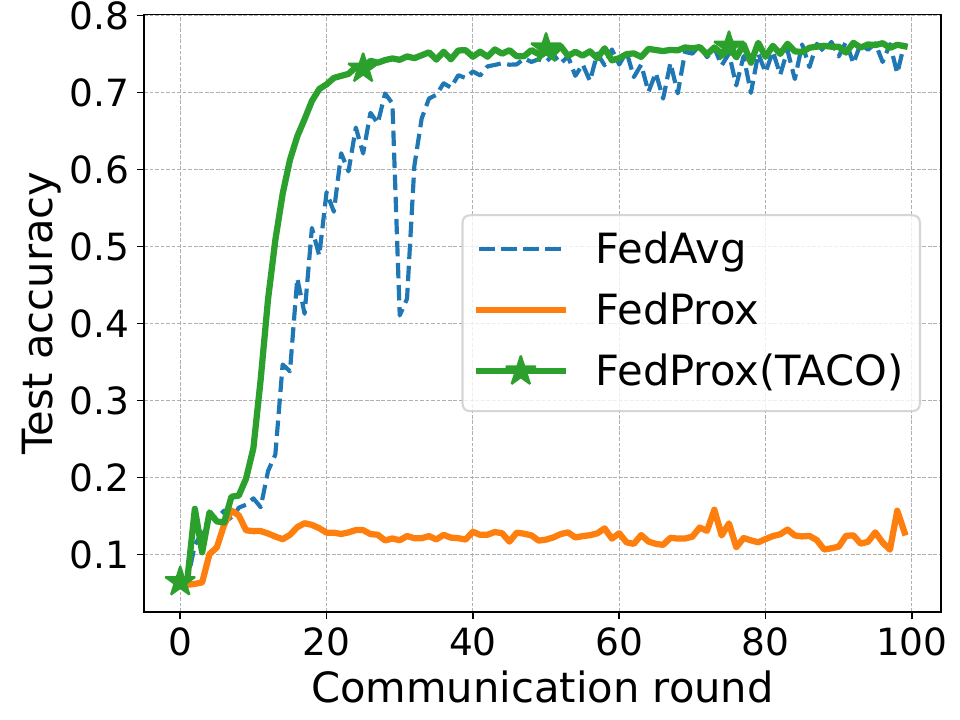}
		\caption{FedProx-SVHN}
	\end{subfigure}
	\begin{subfigure}[t]{0.48\linewidth}
		\includegraphics[width=\textwidth]{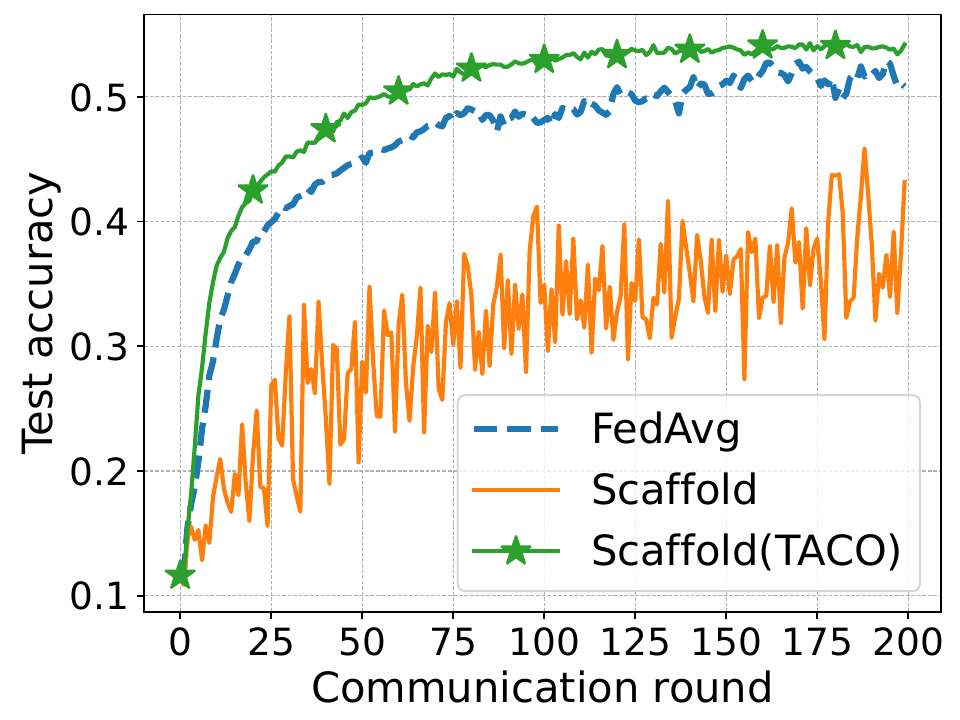}
		\caption{Scaffold-CIFAR10}
	\end{subfigure}
	\caption{Performance gain in prior methods using TACO.}
	\label{fig:extension}
\end{figure}
\subsection{Ablation Study}
Table \ref{tab:ablation} evaluates the contribution of the tailored correction design \eqref{eq:local_update_TACO} and tailored aggregation design \eqref{eq:taco_aggregate} in TACO using FEMNIST and adult datasets across different non-IID settings. We discover that compared to the tailored aggregation design, the tailored correction mechanism proposed in Eq.~\eqref{eq:local_update_TACO} contributes more significantly to the model performance improvement, further demonstrating its importance and superiority.
\subsection{Scalability and Sensitivity Evaluation}
\begin{table}[t]
\centering
\caption{Ablation study (20-client, final test accuracy)}
\begin{tabular}{cccccc}
\toprule
\multirow{2}{*}{\begin{tabular}[c]{@{}c@{}}Tailored\\ Corr.\end{tabular}} & \multirow{2}{*}{\begin{tabular}[c]{@{}c@{}}Tailored\\ Agg.\end{tabular}} & \multicolumn{2}{c}{FEMNIST} & \multicolumn{2}{c}{adult} \\
                                                                          &                                                                          & Dir(0.2)     & Dir(0.5)     & Dir(0.1)    & Dir(0.5)    \\ \hline
\ding{55}                                                                 & \ding{55}                                                                & 89.98\%        & 92.93\%        & 78.42\%       & 78.88\%       \\ 
\ding{55}                                                                 & \ding{51}                                                                & 91.30\%        & 93.20\%        & 79.76\%       & 83.34\%       \\ 
\ding{51}                                                                 & \ding{55}                                                                & 95.16\%        & 95.91\%        & 83.52\%       & 83.69\%             \\ 
\ding{51}                                                                 & \ding{51}                                                                & \textbf{95.22\%}        & \textbf{96.10\%}        & \textbf{83.64\%}       & \textbf{83.80\%}       \\ \bottomrule
\end{tabular}
\label{tab:ablation}
\end{table}
\begin{table}[t]
\centering
\caption{Scalability analysis (100-client)}
\begin{tabular}{cccc}
\toprule
Datasets          & adult & FEMNIST             & CIFAR-100                          \\ \hline
FedAvg    & 76.72±0.43\% & 88.14±0.28\%          & 42.33±0.40\%                   \\
FedProx   & 77.29±1.58\% & 89.54±0.15\%          & 41.83±0.15\%                   \\
FoolsGold & 83.28±0.20\% & 89.66±0.23\%          & 42.39±0.12\%                   \\
Scaffold  & 80.08±0.33\% & 92.26±0.48\%          & 49.77±0.22\%                    \\
STEM      & 77.89±0.47\% & 86.56±0.58\%          & 42.44±0.08\%                   \\
FedACG    & 83.48±0.13\% & 87.90±0.34\%          & 41.99±0.13\%                   \\
TACO      & \textbf{83.71±0.16\%} & \textbf{92.86±0.25\%} & \textbf{53.65±0.13\%}  \\ \bottomrule
\end{tabular}
\label{tab:scalability}
\vspace{-3mm}
\end{table}
\begin{table}[t]
\centering
\caption{Sensitivity of thresholds $\lambda$ and $\kappa$ (FMNIST)}
\begin{tabular}{ccccccc}
\hline
         & \multicolumn{2}{c}{$\lambda=T/10$}                          & \multicolumn{2}{c}{$\lambda=T/5$}                           & \multicolumn{2}{c}{$\lambda=T/2$}                           \\ \hline
$\kappa$ & TPR                           & FPR                         & TPR                           & FPR                         & TPR                           & FPR                         \\
0.4      & 100\%                         & 25\%                        & 100\%                         & 12.5\%                      & 100\%                         & 0\%                         \\
0.5      & 100\%                         & 6.25\%                      & \cellcolor[HTML]{EFEFEF}100\% & \cellcolor[HTML]{EFEFEF}0\% & \cellcolor[HTML]{EFEFEF}100\% & \cellcolor[HTML]{EFEFEF}0\% \\
0.6-0.8  & \cellcolor[HTML]{EFEFEF}100\% & \cellcolor[HTML]{EFEFEF}0\% & \cellcolor[HTML]{EFEFEF}100\% & \cellcolor[HTML]{EFEFEF}0\% & \cellcolor[HTML]{EFEFEF}100\% & \cellcolor[HTML]{EFEFEF}0\% \\
0.9      & \cellcolor[HTML]{EFEFEF}100\% & \cellcolor[HTML]{EFEFEF}0\% & \cellcolor[HTML]{EFEFEF}100\% & \cellcolor[HTML]{EFEFEF}0\% & 0\%                           & 0\%                         \\
1.0      & 0\%                           & 0\%                         & 0\%                           & 0\%                         & 0\%                           & 0\%                         \\ \hline
\end{tabular}
\label{tab:free_thres}
\vspace{-0.5em}
\end{table}

\begin{figure}[t]
    \centering
    \includegraphics[width=0.95\linewidth]{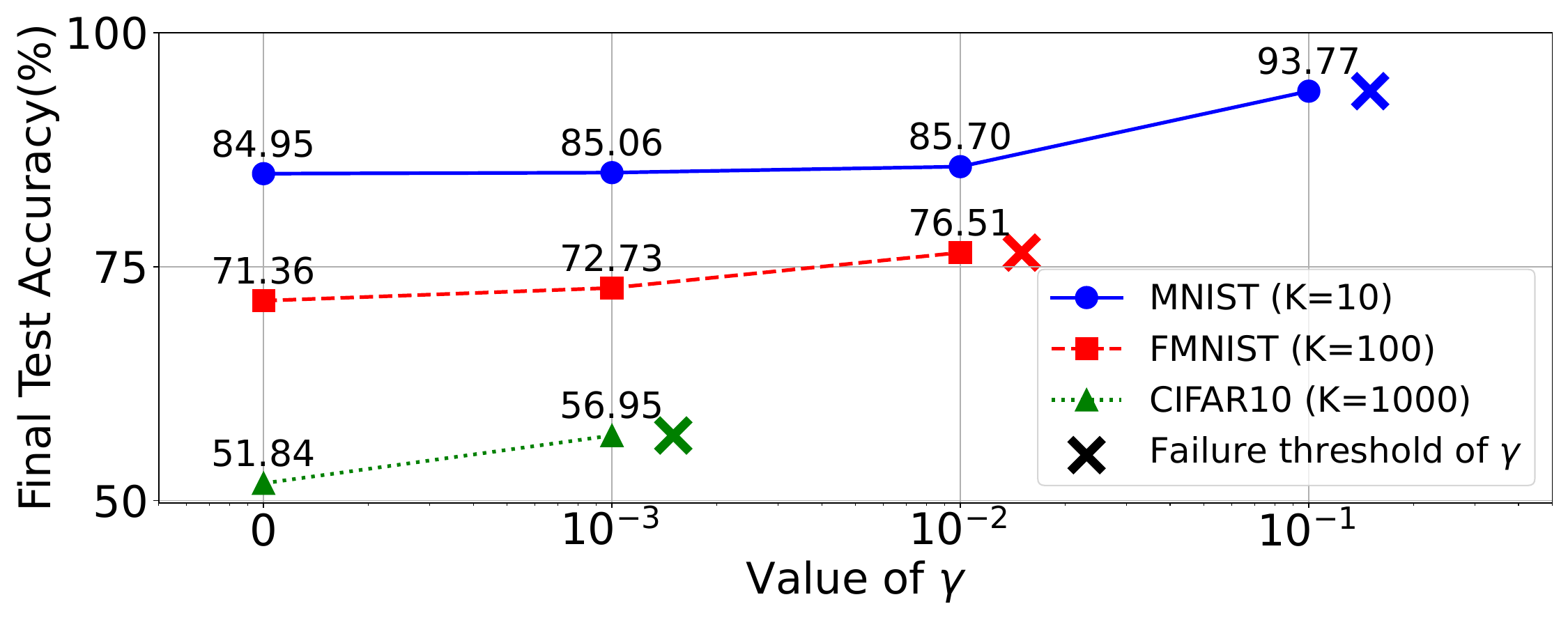}
    \caption{Sensitivity of $\gamma$.}
    \label{fig:sensitivity_gamma}
    \vspace{-1.5em}
\end{figure}

We then evaluate the scalability of our TACO algorithm by conducting large-scale experiments with 100 clients on the adult, FEMNIST, and CIFAR-100 datasets. 
We observe that TACO can consistently outperform existing methods in all datasets and demonstrate pronounced advantage in the larger-scale scenario with 100 clients, showing its great scalability.

We further evaluate the sensitivity of the maximum correction factor $\gamma$ in Eq.~\eqref{eq:local_update_TACO}. The candidate values of $\gamma$ are \{0, 0.001, 0.01, 0.1, 1.0\}. We conduct experiments on three datasets (MNIST, FMNIST, CIFAR-10) using the corresponding appropriate number of local update steps per round, \textit{i.e.}, $K$, for each dataset. Fig.~\ref{fig:sensitivity_gamma} underscores the pivotal role of $\gamma$, showing that larger values of $\gamma$ often offer higher improvement of model correction, thus yielding better model performance. However, an inappropriate large value of $\gamma$ may also lead to convergence failure. Additionally, we discovered 
an inverse relationship between the optimal $\gamma^*$ 
and $K$ ($\gamma^* \approx \frac{1}{K}$). 
This can be explained as that with the increase in $K$, one should select a smaller value of $\gamma$ to avoid over-correction.

We also evaluate the sensitivity of detection thresholds $\lambda$ and $\kappa$ used to identify freeloaders. We review that $\kappa$ is the threshold for determining whether a client is a suspicious freeloader in Eq.~\eqref{eq:lazy-threshold}, and $\lambda$ is the number of times a client is suspected of being a freeloader before being evicted (Line \ref{line:inspect_lazy} in Algorithm \ref{alg:taco}). The candidate values of $\kappa$ are \{0.4, 0.5, 0.6, 0.7, 0.8, 0.9, 1.0\}, and the candidate values of $\lambda$ are \{$T/10, T/5, T/2$\}, where $T$ is the number of communication rounds. As illustrated in Table \ref{tab:free_thres}, an increase in the values of $\kappa$ and $\lambda$ represents a stricter criterion to identify a client as a freeloader, which reduces the probability of falsely categorizing benign clients as freeloaders (lower FPR). However, excessively larger $\kappa$ and $\lambda$ can result in the inability to identify freeloaders (lower TPR). Conversely, smaller $\kappa$ and $\lambda$ lead to a more lenient criterion, increasing the detection rate of freeloaders but also elevating the likelihood of mistakenly classifying benign users as freeloaders (Higher TPR and FPR). To sum up, TACO demonstrates robust freeloader detection capabilities across a wide range of values for $\kappa$ and $\lambda$ (as shown in the shaded regions in Table \ref{tab:free_thres}), and we should avoid excessively large or small values of $\kappa$ and $\lambda$, as they increase the proportion of freeloader clients participating in the FL training, thus undermining the model performance. 

%% file: 8_appendix.tex
\appendix
\allowdisplaybreaks[3]    
To stress the significance of correction factors $1-\alpha_i^t$, we set $\gamma=1$ for brevity, other values of $\gamma$ can also be incorporated.

\begin{proof}[Proof of Theorem \ref{thm:1}]
\setlength{\jot}{1pt}
\begin{align*}
&\mathbb{E}\left[f(\mathbf{z}_{t+1})\right] \\
&=\mathbb{E}\left[f(\mathbf{z}_t)\right]+\mathbb{E}\left[\left<\nabla f(\mathbf{z}_t),\mathbf{z}_{t+1}-\mathbf{z}_t\right>\right]+\frac{L}{2}\mathbb{E}\left[\|\mathbf{z}_{t+1}-\mathbf{z}_t\|^2\right] \\
&=\mathbb{E}\left[f(\mathbf{z}_{t})\right]-\eta_{g}\mathbb{E}\left[\left\langle\nabla f(\mathbf{z}_{t}),\frac{1}{ KN}\sum_{i\in[N],k\in[K]}\nabla f_{i}(\mathbf{w}_{i,k}^{t})\right\rangle\right]\\
&+\frac{L}{2}\mathbb{E}\left[\|\mathbf{z}_{t+1}-\mathbf{z}_{t}\|^{2}\right] \quad \text{(L-smooth and Definition 1)} \\
&=\mathbb{E}\left[f(\mathbf{z}_t)\right]-\eta_g\mathbb{E}\left[\|\nabla f(\mathbf{z}_t)\|^2\right]+\frac{L}{2}\eta_g^2\mathbb{E}\left[\|\tilde{\Delta}_t\|^2\right] \,\text{(Lemma \ref{lem:z_t})}\\
&-\eta_g\mathbb{E}\left[\left\langle\nabla f(\mathbf{z}_t),\frac{1}{KN}\sum_{i\in[N],k\in[K]}(\nabla f_i(\mathbf{w}_{i,k}^t)-\nabla f_i(\mathbf{z}_t))\right\rangle\right]  \\
&\leq\mathbb{E}\left[f(\mathbf{z}_{t})\right]-\eta_g\mathbb{E}\left[\|\nabla f(\mathbf{z}_t)\|^2\right]+\frac{L}{2}\eta_g^2\mathbb{E}\left[\|\tilde{\Delta}_t\|^2\right]\\
&+\frac{\eta_{g}}{2}\mathbb{E}\left[\|\nabla f(\mathbf{z}_{t})\|^{2}+\frac{1}{KN}\sum_{i\in[N]}^{k\in[K]}\|\nabla f_{i}(\mathbf{w}_{i,k}^{t})-\nabla f_{i}(\mathbf{z}_{t})\|^{2}\right]\\
 &\leq\mathbb{E}\left[f(z_{t})\right]-\eta_{g}\mathbb{E}\left[\|\nabla f(z_{t})\|^{2}\right]+\frac{L}{2}\eta_{g}^{2}\mathbb{E}\left[\|\tilde{\Delta}_{t}\|^{2}\right]\\
 &+\frac{\eta_{g}}{2}\mathbb{E}\left[\|\nabla f(z_{t})\|^{2}\right]+\frac{\eta_{g}}{KN}\sum_{i\in[N]}^{k\in[K]}\big(\mathbb{E}\big [\|\nabla f_{i}(\mathbf{w}_{i,k}^{t}) -\nabla f_{i}(\mathbf{w}_{t})\|^{2}\\
 &+\|\nabla f_{i}(\mathbf{w}_{t})-\nabla f_{i}(z_{t})\|^{2}\big ]\big) \quad\quad \text{(Triangle inequality)} \\
&\leq\mathbb{E}\left[f(\mathbf{z}_{t})-\eta_{g}\|\nabla f(\mathbf{z}_{t})\|^{2}+\frac{L}{2}\eta_{g}^{2}\|\Tilde{\Delta}_{t}\|^{2}+\frac{\eta_{g}}{2}\|\nabla f(\mathbf{z}_{t})\|^{2}\right]\\
&+\eta_{g}L^{2}\varepsilon_{t} + \eta_{g}^3L^{2}\left(1-\alpha_t\right)^{2}\mathbb{E}\left[\|\Delta_{t}\|^{2}\right]\quad\text{(L-smooth)} \\
&\leq\underbrace{\mathbb{E}\left[f(\mathbf{z}_{t})\right]-\frac{\eta_g}{2}\mathbb{E}\left[\|\nabla f(\mathbf{z}_{t})\|^{2}\right]+\frac{L\eta_{g}^{2}}{2}\mathbb{E}\left[\|\tilde{\Delta}_{t}\|^{2}\right]+\eta_{g}L^{2}\varepsilon_{t}}_{A}\\& + \eta_g^3L^2\left(1-\alpha_t\right)^{2}\mathbb{E}\left[\|\Delta_{t}\|^{2}\right]
\\& \leq A + \eta_g^3L^2\left(1-\alpha_t\right)^{2}\mathbb{E}\left[\|\Delta_{t}\|^{2}\right]\nonumber \\& = A + \eta_g^3L^2\left(\frac{1}{N}\sum_{i\in\mathcal{N}}(1-\alpha_i^t)\right)^{2}\mathbb{E}\left[\|\frac{1}{KN\eta_l} \sum_{i=1}^{N}\Delta_i^{t-1} \|^{2}\right]\nonumber \\
&\leq A + \frac{\eta_g^3L^2}{K^2N^4\eta_l^2}\left(\sum_{i\in\mathcal{N}}(1-\alpha_i^t)\right)^{2}\left(\sum_{i\in\mathcal{N}}\mathbb{E}\left[\|\Delta_i^{t-1} \|\right]\right)^{2}\nonumber \\
& \text{Using Eq.~(\ref{eq:asmp2-1})/Eq.~(\ref{eq:asmp2-2}), we have } \mathbb{E}\|\Delta_i^{t-1} \|\leq \frac{\mu_i}{c_i}\left\|\nabla f(\mathbf{w}_{t-1})\right\|\nonumber\\
&\leq A + \frac{\eta_g^3L^2}{K^2N^4\eta_l^2}\left(\sum_{i\in\mathcal{N}}(1-\alpha_i^t)\cdot\sum_{i\in\mathcal{N}}\frac{\mu_i}{c_i}\left\|\nabla f(\mathbf{w}_{t-1})\right\| \right)^{2}\\
&\leq A + \eta_g^3\underbrace{\frac{L^2G^2}{K^2N^4\eta_l^2}\left(\sum_{i\in\mathcal{N}}(1-\alpha_i^t)\cdot\sum_{i\in\mathcal{N}}\frac{\mu_i}{c_i} \right)^{2}}_{Y_t} \;\text{(Assumption \ref{asmp3})}\\
& = \mathbb{E}\left[f(\mathbf{z}_{t})\right]-\frac{\eta_g}{2}\mathbb{E}\left[\|\nabla f(\mathbf{z}_{t})\|^{2}\right]+\frac{L\eta_{g}^{2}}{2}\mathbb{E}\left[\|\tilde{\Delta}_{t}\|^{2}\right]\\&+\eta_{g}L^{2}\varepsilon_{t} + \eta_g^3Y_t.
\end{align*}
\end{proof}
\begin{proof}[Proof of Corollary \ref{cor:converge_rate}]
Rearrange the inequality in Theorem \ref{thm:1} and apply the definition of $Y = \mathop{\max}_{t\in[T]} Y_t$, we have:
\begin{align}
    \mathbb{E}\left[\|\nabla f(\mathbf{z}_{t})\|^{2}\right] &\leq \frac{2}{\eta_g}\left(\mathbb{E}\left[f(\mathbf{z}_{t})\right]-\mathbb{E}\left[f(\mathbf{z}_{t+1})\right]\right)+ L\eta_g\|\tilde{\Delta}_{t}\|^{2} \nonumber \\&+ 2L^2\varepsilon_{t} + 2\eta_g^2Y.
    \label{eq:proof_cor2}
\end{align}
According to the the Lemma B.6 derived in the pioneering work \cite{xu2021fedcm}, for a sequence $a_t = \frac{1}{\eta}(b_t - b_{t+1}) + c_1\eta +c_2\eta^2$, where $\{b_t\}_{t\geq0}$ is a non-negative sequence, constants $c_1, c_2\geq 0$, and the learning rate $\eta>0$. Given any training round $T\in\mathcal{N}_{+}$, there always exists a learning rate $\eta$ satisfies:
\begin{equation}
\frac{1}{T}\sum_{t\in[T]}a_t=\mathcal{O}(\sqrt{\frac{b_0c_1}{T}}+\sqrt[3]{\frac{b_0^2c_2}{T^2}}).
\label{eq:lemmab6}
\end{equation}
Applying Eq.~\eqref{eq:lemmab6} into the right side of Eq.~\eqref{eq:proof_cor2}, we derive the convergence rate:
\begin{equation*}
    \frac{1}{T}\sum_{t=1}^{T}\mathbb{E}\left[\|\nabla f(\mathbf{z}_{t})\|^{2}\right] = \mathcal{O}\left(\sqrt{\frac{L}{T}}+\sqrt[3]{\frac{Y}{T^2}}\right).
    \label{eq:convergence_rate_appendix}
\end{equation*}
\end{proof}
\begin{proof}[Proof of Corollary \ref{cor:optimal-alpha}]
    According to the Cauchy-Schwarz inequality, given some scalar $\sigma>0$, minimizing the error term $Y_t = \frac{L^2G^2}{K^2N^4\eta_l^2}\left(\sum_{i\in\mathcal{N}}(1-\alpha_i^t)\sum_{i\in\mathcal{N}}\frac{\mu_i}{c_i} \right)^{2}$ requires: 
\begin{equation*}
    (1-\alpha_i^t)\propto \frac{\mu_i}{c_i}, \left(\text{recall: }\sum\nolimits_{i\in\mathcal{N}}(1-\alpha_i^t)\geq \sigma\right)
\end{equation*}
The error term $Y_t$ in Theorem \ref{thm:1} and $Y$ in Corollary \ref{cor:converge_rate} reach the minimum value. 
\end{proof}

%% file: main.bbl
\begin{thebibliography}{10}
\providecommand{\url}[1]{#1}
\csname url@samestyle\endcsname
\providecommand{\newblock}{\relax}
\providecommand{\bibinfo}[2]{#2}
\providecommand{\BIBentrySTDinterwordspacing}{\spaceskip=0pt\relax}
\providecommand{\BIBentryALTinterwordstretchfactor}{4}
\providecommand{\BIBentryALTinterwordspacing}{\spaceskip=\fontdimen2\font plus
\BIBentryALTinterwordstretchfactor\fontdimen3\font minus \fontdimen4\font\relax}
\providecommand{\BIBforeignlanguage}[2]{{%
\expandafter\ifx\csname l@#1\endcsname\relax
\typeout{** WARNING: IEEEtran.bst: No hyphenation pattern has been}%
\typeout{** loaded for the language `#1'. Using the pattern for}%
\typeout{** the default language instead.}%
\else
\language=\csname l@#1\endcsname
\fi
#2}}
\providecommand{\BIBdecl}{\relax}
\BIBdecl

\bibitem{zhou2019edge}
Z.~Zhou, X.~Chen, E.~Li, L.~Zeng, K.~Luo, and J.~Zhang, ``Edge intelligence: Paving the last mile of artificial intelligence with edge computing,'' \emph{Proceedings of the IEEE}, vol. 107, no.~8, pp. 1738--1762, 2019.

\bibitem{mcmahan2017communication}
B.~McMahan, E.~Moore, D.~Ramage, S.~Hampson, and B.~A. y~Arcas, ``Communication-efficient learning of deep networks from decentralized data,'' in \emph{Proc.~of Artificial intelligence and statistics}, 2017.

\bibitem{liao2023accelerating}
Y.~Liao, Y.~Xu, H.~Xu, Z.~Yao, L.~Wang, and C.~Qiao, ``Accelerating federated learning with data and model parallelism in edge computing,'' \emph{IEEE/ACM Transactions on Networking}, 2023.

\bibitem{dinh2020federated}
C.~T. Dinh, N.~H. Tran, M.~N. Nguyen, C.~S. Hong, W.~Bao, A.~Y. Zomaya, and V.~Gramoli, ``Federated learning over wireless networks: Convergence analysis and resource allocation,'' \emph{IEEE/ACM Transactions on Networking}, vol.~29, no.~1, pp. 398--409, 2020.

\bibitem{xia2024accelerating}
J.~Xia, W.~Wu, L.~Luo, G.~Cheng, D.~Guo, and Q.~Nian, ``Accelerating and securing federated learning with stateless in-network aggregation at the edge,'' in \emph{2024 IEEE 44th International Conference on Distributed Computing Systems (ICDCS)}.\hskip 1em plus 0.5em minus 0.4em\relax IEEE, 2024, pp. 692--702.

\bibitem{li2020federated}
T.~Li, A.~K. Sahu, M.~Zaheer, M.~Sanjabi, A.~Talwalkar, and V.~Smith, ``Federated optimization in heterogeneous networks,'' \emph{Proceedings of Machine learning and systems}, vol.~2, pp. 429--450, 2020.

\bibitem{fung2020limitations}
C.~Fung, C.~J. Yoon, and I.~Beschastnikh, ``The limitations of federated learning in sybil settings,'' in \emph{23rd International Symposium on Research in Attacks, Intrusions and Defenses (RAID 2020)}, 2020, pp. 301--316.

\bibitem{karimireddy2020scaffold}
S.~P. Karimireddy, S.~Kale, M.~Mohri, S.~Reddi, S.~Stich, and A.~T. Suresh, ``Scaffold: Stochastic controlled averaging for federated learning,'' in \emph{International conference on machine learning}.\hskip 1em plus 0.5em minus 0.4em\relax PMLR, 2020, pp. 5132--5143.

\bibitem{khanduri2021stem}
P.~Khanduri, P.~Sharma, H.~Yang, M.~Hong, J.~Liu, K.~Rajawat, and P.~Varshney, ``Stem: A stochastic two-sided momentum algorithm achieving near-optimal sample and communication complexities for federated learning,'' \emph{Advances in Neural Information Processing Systems}, vol.~34, pp. 6050--6061, 2021.

\bibitem{fedacg}
G.~Kim, J.~Kim, and B.~Han, ``Communication-efficient federated learning with accelerated client gradient,'' in \emph{Proceedings of the IEEE/CVF Conference on Computer Vision and Pattern Recognition}, 2024, pp. 12\,385--12\,394.

\bibitem{li2022federated}
Q.~Li, Y.~Diao, Q.~Chen, and B.~He, ``Federated learning on non-iid data silos: An experimental study,'' in \emph{2022 IEEE 38th International Conference on Data Engineering (ICDE)}.\hskip 1em plus 0.5em minus 0.4em\relax IEEE, 2022, pp. 965--978.

\bibitem{wang2020tackling}
J.~Wang, Q.~Liu, H.~Liang, G.~Joshi, and H.~V. Poor, ``Tackling the objective inconsistency problem in heterogeneous federated optimization,'' \emph{Advances in neural information processing systems}, vol.~33, pp. 7611--7623, 2020.

\bibitem{shi2022talk}
D.~Shi, L.~Li, M.~Wu, M.~Shu, R.~Yu, M.~Pan, and Z.~Han, ``To talk or to work: Dynamic batch sizes assisted time efficient federated learning over future mobile edge devices,'' \emph{IEEE Transactions on Wireless Communications}, vol.~21, no.~12, pp. 11\,038--11\,050, 2022.

\bibitem{netzer2011reading}
Y.~Netzer, T.~Wang, A.~Coates, A.~Bissacco, B.~Wu, and A.~Y. Ng, ``Reading digits in natural images with unsupervised feature learning,'' 2011.

\bibitem{fraboni2021free}
Y.~Fraboni, R.~Vidal, and M.~Lorenzi, ``Free-rider attacks on model aggregation in federated learning,'' in \emph{International Conference on Artificial Intelligence and Statistics}.\hskip 1em plus 0.5em minus 0.4em\relax PMLR, 2021, pp. 1846--1854.

\bibitem{zhang2022enabling}
N.~Zhang, Q.~Ma, and X.~Chen, ``Enabling long-term cooperation in cross-silo federated learning: A repeated game perspective,'' \emph{IEEE Transactions on Mobile Computing}, vol.~22, no.~7, pp. 3910--3924, 2022.

\bibitem{lin2019free}
J.~Lin, M.~Du, and J.~Liu, ``Free-riders in federated learning: Attacks and defenses,'' \emph{arXiv preprint arXiv:1911.12560}, 2019.

\bibitem{chen2024toward}
T.~Chen, F.~Wang, W.~Qiu, Q.~Zhang, Z.~Xiong, and Z.~Zheng, ``Toward free-riding attack on cross-silo federated learning through evolutionary game,'' in \emph{2024 IEEE 44th International Conference on Distributed Computing Systems (ICDCS)}.\hskip 1em plus 0.5em minus 0.4em\relax IEEE, 2024, pp. 869--880.

\bibitem{xu2021fedcm}
J.~Xu, S.~Wang, L.~Wang, and A.~C.-C. Yao, ``Fedcm: Federated learning with client-level momentum,'' \emph{arXiv preprint arXiv:2106.10874}, 2021.

\bibitem{bottou2018optimization}
L.~Bottou, F.~E. Curtis, and J.~Nocedal, ``Optimization methods for large-scale machine learning,'' \emph{SIAM review}, vol.~60, no.~2, pp. 223--311, 2018.

\bibitem{wang2019adaptive}
S.~Wang, T.~Tuor, T.~Salonidis, K.~K. Leung, C.~Makaya, T.~He, and K.~Chan, ``Adaptive federated learning in resource constrained edge computing systems,'' \emph{IEEE journal on selected areas in communications}, vol.~37, no.~6, pp. 1205--1221, 2019.

\bibitem{rodio2023federated}
A.~Rodio, F.~Faticanti, O.~Marfoq, G.~Neglia, and E.~Leonardi, ``Federated learning under heterogeneous and correlated client availability,'' in \emph{IEEE INFOCOM 2023-IEEE Conference on Computer Communications}.\hskip 1em plus 0.5em minus 0.4em\relax IEEE, 2023, pp. 1--10.

\bibitem{wang2023federated}
S.~Wang, J.~Perazzone, M.~Ji, and K.~S. Chan, ``Federated learning with flexible control,'' in \emph{IEEE INFOCOM 2023-IEEE Conference on Computer Communications}.\hskip 1em plus 0.5em minus 0.4em\relax IEEE, 2023, pp. 1--10.

\bibitem{liu2023dynamite}
W.~Liu, X.~Zhang, J.~Duan, C.~Joe-Wong, Z.~Zhou, and X.~Chen, ``Dynamite: Dynamic interplay of mini-batch size and aggregation frequency for federated learning with static and streaming dataset,'' \emph{IEEE Transactions on Mobile Computing}, 2023.

\bibitem{kairouz2021advances}
P.~Kairouz, H.~B. McMahan, B.~Avent, A.~Bellet, M.~Bennis, A.~N. Bhagoji, K.~Bonawitz, Z.~Charles, G.~Cormode, R.~Cummings \emph{et~al.}, ``Advances and open problems in federated learning,'' \emph{Foundations and Trends{\textregistered} in Machine Learning}, vol.~14, no. 1--2, pp. 1--210, 2021.

\bibitem{wang2023fedmos}
X.~Wang, Y.~Chen, Y.~Li, X.~Liao, H.~Jin, and B.~Li, ``Fedmos: Taming client drift in federated learning with double momentum and adaptive selection,'' in \emph{IEEE INFOCOM}, 2023.

\bibitem{li2019convergence}
X.~Li, K.~Huang, W.~Yang, S.~Wang, and Z.~Zhang, ``On the convergence of fedavg on non-iid data,'' \emph{arXiv preprint arXiv:1907.02189}, 2019.

\bibitem{misc_adult_2}
B.~Becker and R.~Kohavi, ``{Adult},'' UCI Machine Learning Repository, 1996, {DOI}: https://doi.org/10.24432/C5XW20.

\bibitem{he2016deep}
K.~He, X.~Zhang, S.~Ren, and J.~Sun, ``Deep residual learning for image recognition,'' in \emph{Proceedings of the IEEE conference on computer vision and pattern recognition}, 2016, pp. 770--778.

\bibitem{hu2022autofl}
M.~Hu, W.~Yang, Z.~Luo, X.~Liu, Y.~Zhou, X.~Chen, and D.~Wu, ``Autofl: A bayesian game approach for autonomous client participation in federated edge learning,'' \emph{IEEE Transactions on Mobile Computing}, vol.~23, no.~1, pp. 194--208, 2022.

\bibitem{caldas2018leaf}
S.~Caldas, S.~M.~K. Duddu, P.~Wu, T.~Li, J.~Kone{\v{c}}n{\`y}, H.~B. McMahan, V.~Smith, and A.~Talwalkar, ``Leaf: A benchmark for federated settings,'' \emph{arXiv preprint arXiv:1812.01097}, 2018.

\bibitem{deng2023hierarchical}
Y.~Deng, J.~Ren, C.~Tang, F.~Lyu, Y.~Liu, and Y.~Zhang, ``A hierarchical knowledge transfer framework for heterogeneous federated learning,'' in \emph{IEEE INFOCOM 2023-IEEE Conference on Computer Communications}.\hskip 1em plus 0.5em minus 0.4em\relax IEEE, 2023, pp. 1--10.

\bibitem{shi2023distributionally}
S.~Shi, Y.~Guo, D.~Wang, Y.~Zhu, and Z.~Han, ``Distributionally robust federated learning for network traffic classification with noisy labels,'' \emph{IEEE Transactions on Mobile Computing}, 2023.

\bibitem{jiang2023heterogeneity}
Z.~Jiang, Y.~Xu, H.~Xu, Z.~Wang, and C.~Qian, ``Heterogeneity-aware federated learning with adaptive client selection and gradient compression,'' in \emph{IEEE INFOCOM 2023-IEEE Conference on Computer Communications}.\hskip 1em plus 0.5em minus 0.4em\relax IEEE, 2023, pp. 1--10.

\bibitem{chaninternal}
Y.-H. Chan, R.~Zhou, R.~Zhao, Z.~JIANG, and E.~C. Ngai, ``Internal cross-layer gradients for extending homogeneity to heterogeneity in federated learning,'' in \emph{The Twelfth International Conference on Learning Representations}, 2024.

\bibitem{liu2024can}
J.~Liu, H.~Huang, C.~Wang, R.~Li, T.~Car, Q.~Yang, and Z.~Zheng, ``Can federated learning clients be lightweight? a plug-and-play symmetric conversion module,'' in \emph{2024 IEEE 44th International Conference on Distributed Computing Systems (ICDCS)}.\hskip 1em plus 0.5em minus 0.4em\relax IEEE, 2024, pp. 809--820.

\bibitem{acar2021federated}
D.~A.~E. Acar, Y.~Zhao, R.~M. Navarro, M.~Mattina, P.~N. Whatmough, and V.~Saligrama, ``Federated learning based on dynamic regularization,'' \emph{arXiv preprint arXiv:2111.04263}, 2021.

\bibitem{gao2022feddc}
L.~Gao, H.~Fu, L.~Li, Y.~Chen, M.~Xu, and C.-Z. Xu, ``Feddc: Federated learning with non-iid data via local drift decoupling and correction,'' in \emph{Proceedings of the IEEE/CVF conference on computer vision and pattern recognition}, 2022, pp. 10\,112--10\,121.

\bibitem{ruan2021towards}
Y.~Ruan, X.~Zhang, S.-C. Liang, and C.~Joe-Wong, ``Towards flexible device participation in federated learning,'' in \emph{International Conference on Artificial Intelligence and Statistics}.\hskip 1em plus 0.5em minus 0.4em\relax PMLR, 2021, pp. 3403--3411.

\bibitem{karimireddy2020mime}
S.~P. Karimireddy, M.~Jaggi, S.~Kale, M.~Mohri, S.~J. Reddi, S.~U. Stich, and A.~T. Suresh, ``Mime: Mimicking centralized stochastic algorithms in federated learning,'' \emph{arXiv preprint arXiv:2008.03606}, 2020.

\bibitem{blanchard2017machine}
P.~Blanchard, E.~M. El~Mhamdi, R.~Guerraoui, and J.~Stainer, ``Machine learning with adversaries: Byzantine tolerant gradient descent,'' \emph{Advances in neural information processing systems}, vol.~30, 2017.

\bibitem{haddadpour2021federated}
F.~Haddadpour, M.~M. Kamani, A.~Mokhtari, and M.~Mahdavi, ``Federated learning with compression: Unified analysis and sharp guarantees,'' in \emph{International Conference on Artificial Intelligence and Statistics}.\hskip 1em plus 0.5em minus 0.4em\relax PMLR, 2021, pp. 2350--2358.

\end{thebibliography}
